\documentclass[preprint,review]{elsarticle} 

\usepackage{bm}
\usepackage{enumerate,amsthm,amsmath,amsfonts,amssymb}
\usepackage{mathtools}
\usepackage{xcolor}
\usepackage{graphicx}
\usepackage{algorithm,algorithmic}
\usepackage{subcaption}
\usepackage{tikz}
\usepackage{multirow}
\usepackage{multicol}
\usepackage{hyperref}
\newtheorem{theorem}{Theorem}
\newtheorem{lemma}[theorem]{Lemma}

\newtheorem{corollary}[theorem]{Corollary}
\theoremstyle{definition}
\newtheorem{definition}{Definition}

\newtheorem{assumption}{Assumption}

\theoremstyle{definition}

\newtheorem{remark}{Remark}

\newcommand{\unif}{\epsilon_{\mathrm{unif}}}
\newcommand{\nbb}{\mathbb{N}}
\newcommand{\var}{\mbox{Var}}

\newcommand{\bw}{\mathbf{w}}

\newcommand{\ibb}{\mathbb{I}}

\newcommand{\wcal}{\mathcal{W}}

\newcommand{\zcal}{\mathcal{Z}}

\newcommand{\ebb}{\mathbb{E}}

\newcommand{\bv}{\mathbf{v}}
\newcommand{\rbb}{\mathbb{R}}

\numberwithin{equation}{section}

\usepackage{lineno}

\journal{Applied and Computational Harmonic Analysis}

\begin{document}

\begin{frontmatter}



\title{Minibatch and Local SGD: Algorithmic Stability and Linear Speedup in Generalization}

\author[lei]{Yunwen Lei\corref{mycorrespondingauthor}}
\ead{leiyw@hku.hk}
\author[sun]{Tao Sun}
\ead{suntao.saltfish@outlook.com}
\author[liu]{Mingrui Liu}
\ead{mingruil@gmu.edu}
\cortext[mycorrespondingauthor]{Corresponding author}

\affiliation[lei]{organization={Department of Mathematics},
            addressline={The University of Hong Kong, Pokfulam},
            city={Hong Kong},
            country={China}}
\affiliation[sun]{organization={College of Computer},
            addressline={National University of Defense Technology},
            city={Changsha},
            country={China}}
\affiliation[liu]{organization={Department of Computer Science},
            addressline={George Mason University},
            city={Fairfax},
            country={USA}}

\begin{abstract}%
The increasing scale of data propels the popularity of leveraging parallelism to speed up the optimization. Minibatch stochastic gradient descent (minibatch SGD) and local SGD are two popular methods for parallel optimization. The existing theoretical studies show a linear speedup of these methods with respect to the number of machines, which, however, is measured by optimization errors in a multi-pass setting. As a comparison, the stability and generalization of these methods are much less studied. In this paper, we study the stability and generalization analysis of minibatch and local SGD to understand their learnability by introducing an expectation-variance decomposition. We incorporate training errors into the stability analysis, which shows how small training errors help generalization for overparameterized models. We show minibatch and local SGD achieve a linear speedup to attain the optimal risk bounds.
\end{abstract}

%

\end{frontmatter}



\section{Introduction}

Modern machine learning often comes along with models and datasets of massive scale (e.g., millions or billions of parameters over enormous training datasets)~\citep{zinkevich2010parallelized,li2014efficient,shamir2014distributed,konevcny2016federated}, which renders the training with sequential algorithms impractical for large-scale data analysis.
To speed up the computation, it is appealing to develop learning schemes that can leverage parallelism to reduce the amount of time in the training stage~\citep{woodworth2020local}. First-order stochastic optimization is especially attractive for parallelism since the gradient computation is easy to parallelize across multiple computation devices~\citep{shamir2014distributed,li2022destress,woodworth2020minibatch}. For distributed optimization, communication has been reported to be a major bottleneck for large-scale applications~\citep{stich2018local}. Therefore, increasing the computation to communication ratio is a major concern in developing parallelizable optimization algorithms. 

A simple stochastic first-order method is the minibatch stochastic gradient descent (minibatch SGD)~\citep{shamir2014distributed,dekel2012optimal,cotter2011better,li2014efficient,yin2018gradient}, where the update at each round is performed based on an average of gradients over several training examples  rather than a single example. Using a minibatch helps in reducing the variance, and therefore accelerates the optimization. The computation over a minibatch of size $b$ can be distributed over $M$ machines, where each machine computes a minibatch of size $K=b/M$ before communication. This increases the computation to communication ratio. Due to its simplicity, minibatch SGD has found successful applications in a variety of settings~\citep{woodworth2020minibatch,shamir2014distributed}.

An orthogonal approach to increase the computation to communication ratio is the local SGD~\citep{mcmahan2017communication,stich2018local,yun2021minibatch}. For local SGD with $M$ machines, we divide the implementation into $R$ rounds. At each round, each machine conducts SGD independently in $K$ iterations, after which an average over $M$ machines is taken to get a consensus point. Unlike minibatch SGD, local SGD is constantly improving its behavior even when the machines are not communicating with each other. Due to this appealing property, local SGD has been widely deployed in many applications~\citep{mcmahan2017communication}.

The promising applications of minibatch SGD and local SGD motivate a lot of theoretical work to understand the performance of these methods. A linear speedup with respect to (w.r.t.) the batch size was established for minibatch SGD in both online~\citep{dekel2012optimal} and stochastic setting~\citep{shamir2014distributed,cotter2011better}, which is further extended to its accelerated variants~\citep{dekel2012optimal,woodworth2020minibatch}. The analysis for local SGD is more challenging. A linear speedup w.r.t. the number of machines was developed for local SGD with strongly convex~\citep{stich2018local} and convex problems~\citep{woodworth2020local,khaled2020tighter}. These results on linear speedup build the theoretical foundation for using the parallelism to reduce the computation for large-scale problems.

The above results on linear speedup are obtained for optimization errors in a multi-pass setting, i.e., the performance of models on training examples. However, in machine learning we care more about the generalization behavior of these models on testing examples, which have been scarcely touched for both minibatch and local SGD with multi-passes over the data. To our knowledge, other than regression with the specific least squares loss~\citep{mucke2019beating,carratino2018learning,lin2020optimal,guo2023capacity}, there is no generalization analysis of minibatch and local SGD that shows a linear speedup measured by testing errors. In this paper, we conduct the generalization analysis of minibatch and local SGD based on the concept of algorithmic stability~\citep{bousquet2002stability}. Our aim is to show the linear speedup observed in optimization errors also holds for testing errors. Our main contributions are summarized as follows.


\noindent 1. We develop stability bounds of minibatch SGD for convex, strongly convex, and nonconvex problems. Our stability bounds incorporate the property of small training errors, which are often the case for overparamterized models. For strongly convex problems, we develop stability bounds independent of the iteration number, which is also novel for the vanilla SGD in the sense of removing the Lipschitz continuity assumption.  Based on these stability bounds, we further develop optimistic bounds on excess population risks which imply fast rates under a low noise condition.

\noindent 2. We develop stability bounds of local SGD for both convex and strongly convex problems, based on which we develop excess risk bounds. This gives the first stability and generalization bounds for local SGD.

\noindent 3.   Our risk bounds for both minibatch SGD and local SGD are optimal. For convex problems our bounds are of the order $O(1/\sqrt{n})$, while for $\mu$-strongly convex problems our bounds are of the order $O(1/(n\mu))$, where $n$ is the sample size. These match the existing minimax lower bounds for the statistical guarantees~\citep{agarwal2009information}. Furthermore, we show that minibatch SGD achieves a linear speedup w.r.t. the batch size, and local SGD achieves a linear speedup w.r.t. the number of machines. To our knowledge, these are the first linear speedup for minibatch and local SGD in generalization for general problems in the multi-pass setting. 

To achieve these results, we develop techniques by introducing the \emph{expectation-variance decomposition} and self-bounding property~\citep{kuzborskij2018data,lei2020fine} into the stability analysis based on a reformulation of minibatch SGD with binomial variables~\citep{feldman2019high}.
Indeed, the existing stability analysis of the vanilla SGD~\citep{hardt2016train,kuzborskij2018data,lei2020fine} does not apply to minibatch SGD. Furthermore, even with our formulation, the techniques in~\citep{lei2020fine} would imply suboptimal stability bounds.

The paper is organized as follows. We survey the related work in Section~\ref{sec:work}, and formulate the problem in Section~\ref{sec:problem}. We study the stability and generalization for minibatch SGD in Section \ref{sec:mini}, and extend these discussions to local SGD in Section \ref{sec:local}. We present the proof of minibatch SGD in Section \ref{sec:proof-minibatch-sgd} and the proof of local SGD in Section \ref{sec:proof-local-sgd}. We conclude the paper in Section \ref{sec:conclusion}.

\section{Related Work\label{sec:work}}
In this section, we survey the related work on algorithmic stability, minibatch and local SGD.

\smallskip

\textbf{Algorithmic stability}. As a fundamental concept in statistical learning theory (SLT), algorithmic stability measures the sensitivity of an algorithm w.r.t. the perturbation of a training dataset. Various concepts of stability have been introduced into the literature, including uniform stability~\citep{bousquet2002stability}, hypothesis stability~\citep{bousquet2002stability}, on-average stability~\citep{shalev2010learnability,kuzborskij2018data} and {on-average model stability~\citep{lei2020fine}}. One of the most widely used stability concept is the uniform stability, which can imply almost optimal high-probability bounds~\citep{feldman2019high,bousquet2020sharper,fan2024high}.  Stability has found wide applications in stochastic optimization~\citep{hardt2016train,lei2020fine,kuzborskij2018data,charles2018stability,mou2018generalization,wang2022differentially,christmann2018total,chen2018stability}. An important property of the stability analysis is that it considers only the particular model produced by the algorithm, and therefore can use the property of the learning algorithm to imply capacity-independent generalization bounds. Lower bounds on the stability of gradient methods also draw increasing attention~\citep{bassily2020stability,koren2022benign}. 


\smallskip

\textbf{Minibatch algorithm}. Minibatch algorithms are efficient in speeding up optimization for smooth problems. \citet{shamir2014distributed} showed that minibatch distributed optimization can attain a linear speedup w.r.t. the batch size, which was also observed for general algorithms in an online learning setting~\citep{dekel2012optimal}. These results were improved in~\citep{cotter2011better}, where the convergence rates involve the training error of the best model and would decay fast in an interpolation setting. The above speedup was derived if the batch size is not large. Indeed, a large batch size may negatively affect the performance of the algorithm~\citep{keskar2017large,lin2019don}. Minibatch stochastic approximation methods were studied for stochastic composite optimization problems~\citep{ghadimi2016mini} and nonconvex problems~\citep{gower2021sgd}. Recently, minibatch algorithms have been shown to be immune to the heterogeneity of the problem~\citep{woodworth2020minibatch}.  For problems with nonsmooth loss functions, minibatch algorithms do not get any speedup~\citep{shamir2014distributed}.

\smallskip

\textbf{Local SGD}. Local SGD, also known as ``parallel SGD'' or ``federated averaging'', is widely used to solve large-scale convex and nonconvex optimization problems~\citep{mcmahan2017communication}. A linear speedup in the number ($M$) of machines was obtained for local SGD on strongly convex problems~\citep{stich2018local}. The key observation is that local SGD can roughly yield a reduction in the variance by a factor of $M$. Despite its promising performance in practice, the theoretical guarantees on convergence rates are still a bit weak and are often dominated by minibatch SGD. Indeed, initial analysis of local SGD failed to derive a convergence rate matching minibatch SGD's performance, due to an additional term proportional to the dispersion of the individual machine's iterates for local SGD~\citep{woodworth2020local}.  For example, the work \citep{woodworth2020local}  also presented a lower bound on the performance of local SGD that is worse than the minibatch SGD guarantee in a certain regime, showing that local SGD does not dominate minibatch SGD. Until recently, the guarantees better than minibatch SGD were obtained under some cases (e.g., case with rare communication)~\citep{woodworth2020local,khaled2020tighter,spiridonoff2021communication}. These discussions impose different assumptions: \citet{woodworth2020local} imposed a bounded variance assumption, while \citet{khaled2020tighter} considered an almost sure smoothness assumption without the bounded variance assumption. These results were extended to a heterogeneous distributed learning setting~\citep{khaled2020tighter,woodworth2020minibatch}, for which heterogeneity was shown to be particularly problematic for local SGD. A linear speedup w.r.t. $M$ was also observed for nonconvex loss functions under a more restrictive constraint on the synchronization delay than that in the convex case~\citep{yu2019parallel}. Lower bounds of local SGD were established~\citep{woodworth2020local}. Generalization bounds of federated learning were recently studied based on Rademacher complexity~\citep{mohri2019agnostic} and stability~\citep{sun2023understanding,chen2023minimax}.

\smallskip

The above results on the linear speedup for minibatch and local SGD were obtained for optimization errors, which is the focus of the paper. The benefit of minibatch in generalization was studied for SGD with the square loss function~\citep{mucke2019beating,lin2020optimal,carratino2018learning}. These discussions use the analytic representation of iterators in terms of integral operators, which do not apply to general problems considered here. 

\section{Problem Setup\label{sec:problem}}
Let $\rho$ be a probability measure defined on a sample space $\zcal$, from which we independently draw a dataset $S=\{z_1,\ldots,z_n\}\subset\zcal$ of $n$ examples. Based on $S$, we wish to learn a model $\bw$ in a model space $\wcal=\rbb^d$ for prediction, where $d\in\nbb$ is the dimension. The performance of $\bw$ on a single example $z\in\zcal$ can be measured by a nonnegative loss function $f(\bw;z)$. The empirical behavior of $\bw$ can be quantified by the empirical risk $F_S(\bw):=\frac{1}{n}\sum_{i=1}^{n}f(\bw;z_i).$ Usually, we apply a randomized algorithm $A$ to minimize $F_S$ over $\wcal$ to get a model $A(S)$. Then an algorithm can be considered as a map from the set of samples to $\wcal$, i.e., $A:\cup_{n=1}^\infty\zcal^n\mapsto\wcal$. A good behavior on training examples does not necessarily mean a good behavior on testing examples, which is the quantity of real interest in machine learning and can be quantified by the population risk $F(\bw):=\ebb_Z[f(\bw;Z)]$. Here $\ebb_Z[\cdot]$ denotes the expectation w.r.t. $Z$. In this paper, we study the excess population risk of a model $\bw$ defined by $F(\bw)-F(\bw^*)$, which measures the suboptimality as compared to the best model $\bw^*=\arg\min_{\bw\in\wcal}F(\bw)$.
Our basic strategy is to use the following error decomposition
\begin{equation}
\ebb_{S,A}\big[F(A(S))-F(\bw^*)\big]=\ebb_{S,A}\big[F(A(S))-F_S(A(S))\big]
+\ebb_{S,A}\big[F_S(A(S))-F_S(\bw^*)\big],\label{decomposition}
\end{equation}
where we have used the identity $\ebb_{S,A}[F_S(\bw^*)]=F(\bw^*)$ and {$\ebb_{S,A}[\cdot]$ denotes the expectation w.r.t. $S$ and $A$}. We refer to the first term $\ebb\big[F(A(S))-F_S(A(S))\big]$ as the generalization error (generalization gap), which measures the discrepancy between training and testing at the output model $A(S)$. We call the second term $\ebb\big[F_S(A(S))-F_S(\bw^*)\big]$ the optimization error, which measures the suboptimality in terms of the empirical risk. One can control the optimization error by tools in optimization theory.
As a comparison, there is little work on the generalization error of minibatch SGD and local SGD in the multi-pass setting, the key challenge of which is the dependency of $A(S)$ on $S$. 

In this paper, we will use a specific algorithmic stability ---on-average model stability--- to address the generalization error. We use $\|\cdot\|_2$ to denote the Euclidean norm. We denote $S\sim S'$ if $S$ and $S'$ differ by at most a single example.
\begin{definition}[Uniform Stability\label{def:unif}]
Let $\epsilon>0$. We say a randomized algorithm $A$ is $\epsilon$-uniformly stable if $\sup_{S\sim S',z}\ebb_A[|f(A(S);z)-f(A(S');z)|]\leq \epsilon$.
\end{definition}
\begin{definition}[On-average Model Stability~\citep{lei2020fine}\label{def:aver-stab}]
  Let $S=\{z_1,\ldots,z_n\}$ and $S'=\{z'_1,\ldots,z'_n\}$ be drawn independently from $\rho$. For any $i\in[n]:=\{1,\ldots,n\}$, define
  $S^{(i)}=\{z_1,\ldots,z_{i-1},z'_i,z_{i+1},\ldots,z_n\}$ as the set formed from $S$ by replacing the $i$-th element with $z'_i$. Let $\epsilon>0$.
  We say a randomized algorithm $A$ is $\ell_1$ on-average model $\epsilon$-stable if
  $
  \ebb_{S,S',A}\big[\frac{1}{n}\sum_{i=1}^{n}\|A(S)-A(S^{(i)})\|_2\big]\leq\epsilon,
  $
  and $\ell_2$ on-average model $\epsilon$-stable if
  $
  \ebb_{S,S',A}\big[\frac{1}{n}\sum_{i=1}^{n}\|A(S)-A(S^{(i)})\|_2^2\big]\leq\epsilon^2.
  $
\end{definition}
According to the above definition, on-average model stability considers the perturbation of each single example, and measures how these perturbations would affect the output models on average. Lemma \ref{lem:gen-model-stab} gives a quantitative connection between the generalization error and on-average model stability. We first introduce some necessary definitions. We use $\nabla g$ to denote the gradient of $g$.
\begin{definition}
  Let $g:\wcal\mapsto\rbb$, $G,L>0$ and $\mu\geq0$.
\begin{enumerate}[1.]
  \item We say $g$ is $G$-Lipschitz continuous if $|g(\bw)-g(\bw')|\leq G\|\bw-\bw'\|_2$ for all $\bw,\bw'\in\wcal.$
  \item We say $g$ is $L$-smooth if $\|\nabla g(\bw)-\nabla g(\bw')\|_2\leq L\|\bw-\bw'\|_2$ for all $\bw,\bw'\in\wcal.$
  \item We say $g$ is $\mu$-strongly convex if
  $
  g(\bw)\geq g(\bw')+\langle\bw-\bw',\nabla g(\bw')\rangle+\frac{\mu}{2}\|\bw-\bw'\|_2^2
  $ for all $\bw,\bw'\in\wcal$. We say $g$ is convex if it is $\mu$-strongly convex with $\mu=0$.
\end{enumerate}
\end{definition}
{A non-negative and $L$-smooth function $g$ enjoys the self-bounding property, meaning $\|\nabla g(\bw)\|_2^2\leq2Lg(\bw)$~\citep{srebro2010smoothness}.}
Examples of smooth and convex loss functions include the logistic loss, least square loss and Huber loss. Examples of Lipschitz and convex loss functions include the hinge loss, logistic loss and Huber loss.

\begin{lemma}[\citep{lei2020fine}\label{lem:gen-model-stab}]
Let $S,S'$ and $S^{(i)}$ be constructed as in Definition \ref{def:aver-stab}, and $\gamma>0$.
\begin{enumerate}[(a)]
  \item Suppose for any $z$, the function $\bw\mapsto f(\bw;z)$ is convex. If $A$ is $\ell_1$ on-average model $\epsilon$-stable and $\sup_z\|\nabla f(A(S);z)\|_2\leq G$ for any $S$, then
  $
  \big|\ebb_{S,A}\big[F_S(A(S))-F(A(S))\big]\big|\leq G\epsilon.
  $
  \item Suppose for any $z$, the function $\bw\mapsto f(\bw;z)$ is nonnegative and $L$-smooth.
If $A$ is $\ell_2$ on-average model $\epsilon$-stable, then the following inequality holds
  \[
    \ebb_{S,A}\big[F(A(S))-F_S(A(S))\big]  \leq \frac{L}{\gamma}\ebb_{S,A}\big[F_S(A(S))\big]
    +\frac{L+\gamma}{2n}\sum_{i=1}^{n}\ebb_{S,S',A}\big[\|A(S^{(i)})-A(S)\|_2^2\big].
  \]
  \end{enumerate}
\end{lemma}
Part (a) gives the connection between generalization and $\ell_1$ on-average model stability under a convexity condition, while Part (b) relates generalization to $\ell_2$ on-average model stability under a smoothness condition (without a Lipschitzness condition). Note Part (a) differs slightly from that in \citep{lei2020fine} by replacing the Lipschitz condition with a convexity condition and $\sup_z\|\nabla f(A(S);z)\|_2\leq G$. However, the analysis is almost identical and we omit the proof. An advantage of $\ell_2$ on-average model stability is that the upper bound involves the training errors, and improves if $F_S(A(S))$ is small.

\section{Generalization of Minibatch SGD\label{sec:mini}}
In this section, we consider the minibatch SGD for convex, strongly convex and nonconvex problems. Minibatch SGD is implemented in several rounds/iterations. Let $\bw_1\in\wcal$ be an initial point.
At the $t$-th round, minibatch SGD randomly draws (with replacement) $b$ numbers $i_{t,1},\ldots,i_{t,b}$ independently from the uniform distribution over $[n]$, where $b\in[n]$ is the batch size. Then it updates $\{\bw_t\}$ by ($t\in[R]=\{1,2,\ldots,R\}$)
\begin{equation}\label{sgd}
  \bw_{t+1}=\bw_t-\frac{\eta_t}{b}\sum_{j=1}^{b}\nabla f(\bw_t;z_{i_{t,j}}),
\end{equation}
where $\{\eta_t\}$ is a positive step size sequence. If $b=1$, then Eq. \eqref{sgd} recovers the vanilla SGD.
If $b=n$, the above scheme is still different from gradient descent since we consider selection with replacement. For simplicity, we always assume $b\geq2$.
We summarize the results of minibatch SGD in Table \ref{tab:mini}. 
\begin{table*}[ht]
  \caption{Excess population risks of Minibatch SGD for convex, strongly convex and gradient-dominated problems. We consider smooth problems and only show the dependency on $n,b,\mu$ and $F(\bw^*)$. The column ``Risk'' denotes the excess population risk, the column ``$R$'' denotes the number of iterations, the column ``Constraint'' indicates the constraint on the batch size $b$ and the column ``Optimal $R$'' is derived by putting the largest $b$ in $R$.
  We achieve a linear speedup w.r.t. the batch size for convex, strongly convex and nonconvex problems (PL condition is defined in Eq. \eqref{pl}). For convex problems, we derive optimistic bounds which improve to $O(n^{-1})$ in a low noise case, i.e., $F(\bw^*)< n^{-1}$.\label{tab:mini}}

\centering\renewcommand{\arraystretch}{1.6}
  \begin{tabular}{|c|c|c|c|c|c|}
    \hline
    \multicolumn{2}{|c|}{Assumption} & Risk & $R$ & Constraint  & Optimal $R$  \\ \hline
    \multirow{2}{*}{convex} & $F(\bw^*)\geq 1/n$  & $\sqrt{F(\bw^*)/n}$ & $n/b$ & $b\leq \frac{\sqrt{nF(\bw^*)}}{2L}$  &$\frac{\sqrt{n}}{\sqrt{F(\bw^*)}}$  \\ \cline{2-6}
    & {$F(\bw^*)< 1/n$} & $\frac{1}{n}$ & $n$ & ---  & $n$  \\ \hline
    \multicolumn{2}{|c|}{$\mu$-strongly convex} & $1/(n\mu)$ & $\max\{n/b,\mu^{-1}\log n\}$ & --- & $\mu^{-1}\log n$  \\ \hline
    \multicolumn{2}{|c|}{$\mu$-PL condition} & $1/(n\mu)$ & $n/(b\mu^2)$ & $b\leq \sqrt{n}/\mu$ & $\mu^{-1}\log n$  \\
    \hline
  \end{tabular}
\end{table*}

\subsection{Convex Case}
We first present stability bounds to be proved in Section \ref{sec:proof-on-average}. Eq. \eqref{on-average-a} considers the  $\ell_1$ on-average model stability, while Eq. \eqref{on-average-b} considers the $\ell_2$ on-average model stability. An advantage of the analysis with $\ell_2$ on-average model stability over $\ell_1$ on-average model stability is that it can imply generalization bounds without a Lipschitzness condition. We denote $A\lesssim B$ if there exists a universal constant $C$ such that $A\leq CB$. We denote $A\gtrsim B$ if there exists a universal constant $C$ such that $A\geq CB$. We denote $A\asymp B$ if $A\lesssim B$ and $A\gtrsim B$.

\begin{theorem}[Stability Bounds for Minibatch SGD: Convex Case\label{thm:on-average}]
  Assume for all $z\in\zcal$, the map $\bw\mapsto f(\bw;z)$ is nonnegative, convex and $L$-smooth.
  Let $S,S'$ and $S^{(m)}$ be given in Definition \ref{def:aver-stab}. Let $\{\bw_t\}$ and $\{\bw_t^{(m)}\}$ be produced by \eqref{sgd} with $\eta_t\leq2/L$ based on $S$ and $S^{(m)}$, respectively. Then
  \begin{equation}\label{on-average-a}
    \frac{1}{n}\sum_{m=1}^{n}\ebb[\|\bw_{t+1}-\bw_{t+1}^{(m)}\|_2] \leq
    \sum_{k=1}^{t}\frac{2\eta_k\sqrt{2L\ebb[F_S(\bw_k)]}}{n}
  \end{equation}
  and
  \begin{equation}\label{on-average-b}
\frac{1}{n}\sum_{m=1}^{n}\ebb\big[\|\bw_{t+1}-\bw_{t+1}^{(m)}\|_2^2\big]
  \leq \frac{16L}{nb}\sum_{k=1}^{t}\eta_k^2\ebb\big[F_S(\bw_k)]+\frac{8}{n^3}\sum_{m=1}^{n}\ebb\Big[\Big(\sum_{k=1}^{t}\eta_k\|\nabla f(\bw_k;z_m)\|_2\Big)^2\Big].
  \end{equation}
\end{theorem}
\begin{remark}[Explanation and comparison]
A property of these stability bounds is that they involve the empirical risks of $\bw_k$, which would be small since we are minimizing the empirical risk by stochastic optimization algorithms.
Similar stability bounds involving $F_S(\bw_k)$ were developed for the vanilla SGD~\citep{lei2020fine}. Their argument needs to distinguish two cases according to whether the algorithm chooses a particular example at each iteration. This argument does not work for the minibatch SGD since we draw $b$ examples per iteration, and we can draw the particular example several times. We bypass this difficulty by introducing the \emph{expectation-variance decomposition} and self-bounding property into the stability analysis based on a reformulation of minibatch SGD~\citep{kuzborskij2018data,lei2020fine,feldman2019high}. We refer the readers to Remark \ref{rem:novelty} for the detailed discussions on the novelty of our analysis.

The stability of minibatch SGD with $\eta_t=\eta$ has also been studied recently~\citep{yin2018gradient,bassily2019private}. The discussions in Theorem 9 in~\citep{yin2018gradient} give a stability bound of the order $O(\eta t/n+\gamma\eta t)$,  where $\gamma=\mathrm{Pr}\{\inf_{\bw,\bw'}\bar{B}_S(\bw,\bw')<(b-1)/(2/(L\eta)-n/(n-1))\}$ and $\bar{B}_S(\bw,\bw')$ is a measure on the gradient diversity defined below
\[
\bar{B}_S(\bw,\bw'):=\frac{n\sum_{i=1}^{n}\|\nabla f(\bw;z_i)-\nabla f(\bw';z_i)\|_2^2}{\|\sum_{i=1}^{n}(\nabla f(\bw;z_i)-\nabla f(\bw';z_i))\|_2^2}.
\]If $\gamma$ is not very small, their stability bounds would be vacuous due to the term $\gamma\eta t$. The stability bound order $O(\eta t/n)$ was developed in~\citep{bassily2019private}. These discussions require $f$ to be convex, smooth and Lipschitz continuous.  Furthermore, these discussions do not incorporate training errors into the stability bounds, and cannot imply optimistic bounds. We remove the Lipschitz condition in our analysis and obtain optimistic bounds.
\end{remark}

We plug the stability bounds in Theorem \ref{thm:on-average} into Lemma \ref{lem:gen-model-stab} to control generalization errors, which together with the optimization error bounds in Lemma \ref{lem:opt-cotter}, implies the following excess risk bounds. It should be noted that we do not require the function $f$ to be Lipschitz continuous. The proof is given in Section \ref{sec:proof-risk-mini}.
\begin{theorem}[Risk Bounds for Minibatch SGD: Convex Case]\label{thm:risk-mini}
  Assume for all $z\in\zcal$, the function $\bw\mapsto f(\bw;z)$ is nonnegative, convex and $L$-smooth.
  Let $\{\bw_t\}$ be produced by \eqref{sgd} with  $\eta_t=\eta\leq1/(2L)$. Then the following inequality holds for $\bar{\bw}_{R}:=\frac{1}{R}\sum_{t=1}^{R}\bw_t$ and any $\gamma>0$
  \[
    \ebb[F(\bar{\bw}_{R})]\!-\!F(\bw^*)\lesssim \frac{\eta LF(\bw^*)}{b}\!+\!\frac{\|\bw^*\|_2^2}{\eta R}
    +
    L\Big(F(\bw^*)\!+\!\frac{\|\bw^*\|_2^2}{\eta R}\Big)\bigg(\frac{1}{\gamma}+(L\!+\!\gamma)\eta^2\Big(\frac{R}{nb}\!+\!\frac{R^2}{n^2}\Big)\bigg).
  \]
\end{theorem}
Note the above excess risk bounds involve $F(\bw^*)$ and would improve if $F(\bw^*)$ is small, which is true in many learning problems. The terms involving $F(\bw^*)$ also correspond to gradient noise since the variance of gradients can be bounded by function values according to the self-bounding property of smooth functions. The risk bounds of this type are called optimistic bounds in the literature~\citep{srebro2010smoothness}.

As a corollary, we develop explicit excess risk bounds by choosing suitable step sizes and number of rounds, using the idea of early-stopping~\citep{yao2007early}. Note the step size depends on $F(\bw^*)$ which is unknown to us. However, this is not a big issue since we can choose step sizes independent of $F(\bw^*)$ to derive bounds of the same order of $n$ but worse order of $F(\bw^*)$. It shows that minibatch SGD can achieve the excess risk bounds of the order $\sqrt{F(\bw^*)/n}$ if $F(\bw^*)\geq 1/n$, and can imply much better error bounds of the order $1/n$ if $F(\bw^*)<1/n$. The proof is given in Section \ref{sec:proof-risk-mini}.

\begin{corollary}\label{cor:mini}
  Let assumptions in Theorem \ref{thm:risk-mini} hold and $\eta=\min\big\{\frac{\|\bw^*\|_2b}{\sqrt{LnF(\bw^*)}},\frac{1}{2L}\big\}$.
  \begin{enumerate}[1.]
    \item If $F(\bw^*)\geq 4Lb^2\|\bw^*\|_2^2/n$, we can take $R\asymp\frac{n}{b}$ to derive
    $\ebb[F(\bar{\bw}_{R})]-F(\bw^*)\lesssim \frac{(LF(\bw^*))^{\frac{1}{2}}\|\bw^*\|_2}{\sqrt{n}}$.
    \item If $F(\bw^*)\leq 4Lb^2\|\bw^*\|_2^2/n$, we take $R\asymp n$ to get
    $
    \ebb[F(\bar{\bw}_{R})]\lesssim F(\bw^*)+\frac{L\|\bw^*\|_2^2}{n}.
    $
  \end{enumerate}
\end{corollary}

\begin{remark}[Linear speedup]
  We now give some explanations on linear speedup. For the case $F(\bw^*)\gtrsim 1/n$, a larger batch size allows for a larger step size, which further decreases the number $R$ of rounds. It shows that minibatch SGD achieves a linear speedup if the batch size is not large, i.e., it only requires $O(n/b)$ rounds to achieve the excess risk bound $O(n^{-\frac{1}{2}})$ if $b\lesssim \sqrt{nF(\bw^*)}/(\sqrt{L}\|\bw^*\|_2)$. Such a linear speedup was observed for optimization errors for multi-pass SGD~\citep{cotter2011better}. Indeed, it was shown that minibatch SGD requires $O(n/b)$ rounds to achieve the optimization error bounds $\ebb[F_S(\bar{\bw}_R)]-F_S(\bw^*)\lesssim \sqrt{F_S(\bw^*)/n}$ if $b\lesssim \sqrt{nF(\bw^*)}/(\sqrt{L}\|\bw^*\|_2)$. We extend the existing optimization error analysis to generalization, and develop the first linear speedup of the minibatch multi-pass SGD as measured by risks for general convex problems. In particular, our regime $b\lesssim \sqrt{nF(\bw^*)}$ for linear speedup in generalization matches the regime $b\lesssim \sqrt{nF_S(\bw^*)}$ for the linear speedup in optimization~\citep{cotter2011better}.

  For the case $F(\bw^*)\lesssim 1/n$, Corollary \ref{cor:mini} shows that a larger batch size does not bring any gain in speeding up the risk bounds. The underlying reason is that the variance is already very small in this case, and a further reduction of variance by minibatch does not bring essential benefits in the learning process. 
\end{remark}

\subsection{Strongly Convex Case}
We now consider strongly convex problems. Theorem \ref{thm:stab-mini-sg} gives  stability bounds, while Theorem \ref{thm:risk-mini-sg} gives excess population risk bounds. The proofs are given in Section \ref{sec:proof-mini-sg}.
\begin{theorem}[Stability Bounds for Minibatch SGD: Strongly Convex Case\label{thm:stab-mini-sg}]
  Assume for all $z\in\zcal$, the map $\bw\mapsto f(\bw;z)$ is nonnegative, $\mu$-strongly convex and $L$-smooth.
  Let $S,S'$ and $S^{(m)}$ be constructed as in Definition \ref{def:aver-stab}. Let $\{\bw_t\}$ and $\{\bw_t^{(m)}\}$ be produced by \eqref{sgd} based on $S$ and $S^{(m)}$, respectively. Then
  \begin{equation}\label{stab-mini-sg-a}
    \frac{1}{n}\sum_{m=1}^{n}\ebb[\|\bw_{t+1}-\bw_{t+1}^{(m)}\|_2]  \leq
    \frac{2\sqrt{2L}}{n}\sum_{k=1}^{t}\eta_k\sqrt{\ebb[F_S(\bw_k)]}\prod_{k'=k+1}^{t}(1-\mu\eta_{k'}/2),
  \end{equation}
  \begin{equation}
    \frac{1}{n}\sum_{m=1}^{n}\ebb[\|\bw_{t+1}-\bw_{t+1}^{(m)}\|_2] \lesssim 1/(n\mu),\label{stab-mini-sg-b}
  \end{equation}
  \begin{equation}\label{stab-mini-sg-c}
\frac{1}{n}\sum_{m=1}^{n}\ebb\big[\|\bw_{t+1}-\bw_{t+1}^{(m)}\|_2^2\big]\leq
\sum_{k=1}^{t}\Big(\frac{16L\eta_k^2}{nb}+\frac{32L\eta_k}{n^2\mu}\Big)\ebb[F_S(\bw_k)]\prod_{k'=k+1}^{t}(1-\mu\eta_{k'}/2).
  \end{equation}
\end{theorem}
\begin{remark}[Explanation]
  Eq. \eqref{stab-mini-sg-a} and Eq. \eqref{stab-mini-sg-b} consider the $\ell_1$ on-average stability. The former
   involves the empirical risks in the upper bound and therefore can benefit from small empirical risks, while the latter  shows minibatch SGD is always stable in the strongly convex case, no matter how many iterations it takes. Eq. \eqref{stab-mini-sg-b} is also new in the vanilla SGD case with $b=1$. Indeed, the work \citep{hardt2016train} also derived the iteration-independent stability bound $O(1/n\mu)$. However, their discussion requires the function $f$ to be strongly-convex, smooth and Lipschitz. We show that the Lipschitz condition can be removed without affecting the stability bounds. Eq. \eqref{stab-mini-sg-c} addresses the $\ell_2$ on-average stability, which shows that increasing the batch size is beneficial to stability. 
\end{remark}

\begin{theorem}[Risk Bounds for Minibatch SGD: Strongly Convex Case\label{thm:risk-mini-sg}]
  Let assumptions in Theorem \ref{thm:stab-mini-sg} hold and assume $\sup_z\|\nabla f(A(S);z)\|_2\leq G$. Let $\sigma_*^2=\ebb_{i_t}[\|\nabla f(\bw^*;z_{i_t})\|_2^2]$. If $R\geq \frac{L}{\mu}\log\frac{nL}{G}$ and $b\geq n\sigma_*^2/(GR)$, then we can find appropriate step size sequences and an average $\hat{\bw}_R$ of $\{\bw_t\}_{t=1}^R$ such that $\ebb[F(\hat{\bw}_{R})]-F(\bw^*)\lesssim G/(n\mu)$.
\end{theorem}
  Note that the assumption $\sup_z\|\nabla f(A(S);z)\|_2\leq G$ is much milder than the Lipschitz condition since it only requires a bound of the gradient on the output model, which can be achieved by a projection to the final output.
  To obtain the excess population risk bounds of the order $O(G/(n\mu))$, we require $R=\max\{\frac{n\sigma_*^2}{Gb},\frac{L}{\mu}\log\frac{nL}{G}\}$. Then, if $b\lesssim n\mu\sigma_*^2/(GL\log (nL/G))$, we know $\frac{L}{\mu}\log\frac{nL}{G}\lesssim \frac{n\sigma_*^2}{Gb}$ and choose $R\asymp \frac{n\sigma_*^2}{Gb}$ to obtain a linear speedup w.r.t. the batch size.
\subsection{Nonconvex Case}
In this subsection, we consider minibatch SGD for nonconvex problems. The following theorem presents the stability bounds for smooth problems without the convexity and Lipschitzness assumption. The proof is given in Section \ref{sec:stab-nonconvex}.
\begin{theorem}\label{thm:stab-mini-nonconvex}
Assume for all $z\in\zcal$, the map $\bw\mapsto f(\bw;z)$ is nonnegative and $L$-smooth.
Let $S,S'$ and $S^{(m)}$ be given in Definition \ref{def:aver-stab}. Let $\{\bw_t\}$ and $\{\bw_t^{(m)}\}$ be produced by \eqref{sgd} with $\eta_t\leq2/L$ based on $S$ and $S^{(m)}$, respectively. Then
\[
\frac{1}{n}\sum_{m=1}^{n}\ebb[\|\bw_{t+1}-\bw_{t+1}^{(m)}\|_2]\leq \frac{2\sqrt{2L}}{n}\sum_{k=1}^{t}\eta_k\ebb\big[\sqrt{F_S(\bw_k)}\big]\prod_{k'=k+1}^{t}(1+\eta_{k'}L).
\]
\end{theorem}

Now, we consider a special nonconvex problem under a Polyak-{\L}ojasiewicz (PL) condition. The PL condition  was shown to hold for deep (linear) and shallow neural networks~\citep{charles2018stability}.
\begin{assumption}[Polyak-{\L}ojasiewicz Condition\label{ass:pl}]
  Let $\bw_S=\arg\min_{\bw\in\wcal}F_S(\bw)$.
  We assume $F_S$ satisfies the PL condition with parameter $\mu>0$, i.e., for all $\bw\in\wcal$
  \begin{equation}\label{pl}
  \ebb_S\big[F_S(\bw)-F_S(\bw_S)\big]\leq \frac{1}{2\mu}\ebb_S\big[\|\nabla F_S(\bw)\|_2^2\big].
  \end{equation}
\end{assumption}

Theorem \ref{thm:risk-mini-pl} gives risk bounds for minibatch SGD under the PL condition, whose proof is given in Section \ref{sec:stab-nonconvex}.
\begin{theorem}[Risk Bounds for Minibatch SGD: PL Condition\label{thm:risk-mini-pl}]
Assume for all $z\in\zcal$, the map $\bw\mapsto f(\bw;z)$ is nonnegative and $L$-smooth. Let $\{\bw_t\}$ be produced by Eq. \eqref{sgd} with $\eta_t=2/(\mu(t+a))$ and $a\geq4L/\mu$.  Let Assumption \ref{ass:pl} hold and $\ebb_z\big[\|\nabla f(\bw_t;z_{i_k})-\nabla F_S(\bw_t)\|_2^2\big]\leq\sigma^2$, where ${i_k}$ follows from the uniform distribution over $[n]$. If $R\geq\max\big\{L\sqrt{n}/\mu,L\sigma^2n/(b\mu^2)\big\}$, then
$
\ebb[F(\bw_R)]-F(\bw^*)\lesssim L/(n\mu).
$
\end{theorem}

According to Theorem~\ref{thm:risk-mini-pl},  we require $R\geq\max\big\{L\sqrt{n}/\mu,L\sigma^2n/(b\mu^2)\big\}$ to obtain the excess risk bounds $O(1/(n\mu))$. If $b\leq \sigma^2\sqrt{n}/\mu$, we have $L\sigma^2n/(b\mu^2)\geq L\sqrt{n}/\mu$ and therefore we can choose $R\asymp L\sigma^2n/(b\mu^2)$ to obtain a linear speedup w.r.t. the batch size. In particular, we can choose $b\asymp \sigma^2\sqrt{n}/\mu$ and $R\asymp L\sqrt{n}/\mu$ to get the bound $\ebb[F(\bw_R)]-F(\bw^*)\lesssim L/(n\mu)$.

\section{Generalization of Local SGD\label{sec:local}}
In this section, we consider local SGD with $M$ machines and $R$ rounds. At the $r$-th round, each machine starts with the same iterate $\bw_{r}$ and independently applies SGD with $K$ steps. After that, we take an average of the iterates in each machine to get a consensus point $\bw_{r+1}$. Let $\bw_{m,r,t+1}$ be the $(t+1)$-th iterate in the machine $m$ at round $r$. Then, the formulation of local SGD is given below
\begin{align}
  \bw_{m,r,1}&=\bw_r,\quad m\in[M],\notag\\
  \bw_{m,r,t+1}&=\bw_{m,r,t}-\eta_{r,t}\nabla f(\bw_{m,r,t};z_{i_{m,r,t}}),\quad t\in[K],\notag\\
  \bw_{r+1}&=\frac{1}{M}\sum_{m=1}^{M}\bw_{m,r,K+1},\quad r\in[R],\label{local-sgd}
\end{align}
where $\eta_{r,t}$ is the step size for the $t$-th update at round $r$, and $i_{m,r,t}$ is drawn independently from the uniform distribution over $[n]$. The pseudo-code is given in Algorithm \ref{alg:local}. If $R=1$, then local SGD becomes the one-shot SGD, i.e., one only takes an average once in the end of the optimization~\citep{zinkevich2010parallelized,lin2017distributed,hu2020distributed}. If $K=1$, then local SGD becomes the minibatch SGD. Note that the computation cost per machine is $KR$. We summarize the results on local SGD in Table \ref{tab:local-sgd}, where we consider smooth problems and ignore constant factors.

\begin{algorithm}
	\caption{Local SGD\label{alg:local}}
	\begin{algorithmic}[1]
\STATE {\bf Inputs:} step sizes $\{\eta_{m,r,t}\}$ and $S$
    \STATE {\bf Initialize:} $\bw_1\in\wcal$
		\FOR {$r=1,2,\ldots,R$}
			\FOR {$m=1,2,\ldots,M$ \textbf{in parallel}}
                \STATE $\bw_{m,r,1}=\bw_r$
                \FOR {$t=1,2,\ldots, K$}
				 \STATE $\bw_{m,r,t+1}=\bw_{m,r,t}-\eta_{r,t}\nabla f(\bw_{m,r,t};z_{i_{m,r,t}})$ 
                \ENDFOR
			\ENDFOR
        \STATE $\bw_{r+1}=\frac{1}{M}\sum_{m=1}^{M}\bw_{m,r,K+1}$
		\ENDFOR
    \STATE {\bf Outputs:} an average of $\bw_{m,r,t}$
	\end{algorithmic}
\end{algorithm}

%

\begin{table*}[ht]
  \caption{Excess population risks of Local SGD for convex and strongly convex problems.  The column ``Risk'' denotes the excess population risk, the column ``$KR$'' denotes the number of iterations per local machine, the column ``$R$'' denotes the communication cost, the column ``Constraint'' indicates the constraint on the number of machines $M$ and the column ``Optimal $KR$'' is derived by putting the largest $M$ in $KR$.
  We achieve a linear speedup w.r.t. the number of machines for both convex and strongly convex problems, under different regimes of $M$. \label{tab:local-sgd}}

\centering\renewcommand{\arraystretch}{1.6}
  \begin{tabular}{|c|c|c|c|c|c|}
    \hline
    Assumption & Risk & $KR$ & $R$ & Constraint & Optimal $KR$  \\ \hline
    convex & $O(1/\sqrt{n})$ & $n/M$ & $n/(KM)$ & $M\leq n^{\frac{1}{2}}$ & $\sqrt{n}$  \\ \hline
    $\mu$-strongly convex & $O((n\mu)^{-1}\log (KR))$ & $n/M$ &$n/(KM)$& $M\leq \sqrt{n\mu}$ & $\sqrt{n/\mu}$ \\
    \hline
  \end{tabular}
\end{table*}
In the following theorem, we develop the stability bounds for local SGD to be proved in Section \ref{sec:proof-stab-local}. We consider both $\ell_1$ and $\ell_2$ on-average model stabilities. 

\begin{theorem}[Stability Bound for Local SGD\label{thm:stab-local}]
Assume for all $z\in\zcal$, the map $\bw\mapsto f(\bw;z)$ is nonnegative, convex and $L$-smooth.
Let $S,S'$ and $S^{(k)}$ be constructed as in Definition \ref{def:aver-stab}. Let $\{\bw_r\}$ and $\{\bw_r^{(k)}\}$ be produced by \eqref{local-sgd} with $\eta_{r,t}\leq2/L$ based on $S$ and $S^{(k)}$, respectively. Then
\begin{equation}\label{local-l1}
  \frac{1}{n}\sum_{k=1}^{n}\ebb\big[\|\bw_{R+1}-\bw^{(k)}_{R+1}\|_2\big]\leq
  \frac{2\sqrt{2L}}{nM}\sum_{r=1}^{R}\sum_{m=1}^{M}\sum_{t=1}^{K}\eta_{r,t}\ebb\Big[\sqrt{F_S(\bw_{m,r,t})}\Big],
\end{equation}
  \begin{multline}
\frac{1}{n}\sum_{k=1}^{n}\ebb\big[\big\|\bw_{R+1}-\bw_{R+1}^{(k)}\big\|_2^2\big]\leq \frac{16L}{nM^2}\sum_{r=1}^{R}\sum_{m=1}^{M}\sum_{t=1}^{K}\eta^2_{r,t}\ebb\big[F_S(\bw_{m,r,t})\big] \\ +
  \frac{2}{n^3M^2}\sum_{k=1}^{n}\ebb\Big[\Big(\sum_{r=1}^{R}\sum_{m=1}^{M}\sum_{t=1}^{K}\eta_{r,t}\|\nabla f(\bw_{m,r,t};z_k)-\nabla f(\bw_{m,r,t}^{(k)};z_k')\|_2\Big)^2\Big].\label{local-l2}
\end{multline}
\end{theorem}
\begin{remark}[Simplification]
  Note that the above stability bounds involve empirical risks, and can benefit from small empirical risks. Assume $\eta_{r,t}=\eta$ and $\ebb\big[\sqrt{F_S(\bw_{m,r,t})}\big]\lesssim 1$ (this is a reasonable assumption since we are minimizing $F_S$). Then Eq. \eqref{local-l1} implies $\frac{1}{n}\sum_{k=1}^{n}\ebb\big[\|\bw_{R+1}-\bw^{(k)}_{R+1}\|_2\big]\lesssim KR\eta/n$. Eq. \eqref{local-l2} implies
  $
  \frac{1}{n}\sum_{k=1}^{n}\ebb\big[\big\|\bw_{R+1}-\bw_{R+1}^{(k)}\big\|_2^2\big]\lesssim KR\eta^2/(nM)+R^2K^2\eta^2/n^2,
  $
  which shows that increasing the number of machines improves the stability and generalization. It was shown that increasing $M$ can improve the optimization~\citep{woodworth2020local}. For example, the optimization error bound of the order $O\big(\frac{1}{K^{\frac{1}{3}}R^{\frac{2}{3}}}+\frac{1}{\sqrt{MKR}}\big)$ was developed in~\citep{woodworth2020local}. Therefore, we expect that increasing $M$ would accelerate the learning process.
\end{remark}
\begin{remark}[Effect of $M$\label{rem:number-machine}]
  We give some explanation on the effect of $M$ on stability analysis.
  Note the above $\ell_1$ on-average stability bounds are independent of $M$, while the $\ell_2$ on-average stability bounds improve as $M$ increases. These phenomena can be explained by how the average operator affects the expectation and variance. Indeed, both the $\ell_1$  and $\ell_2$ stability analysis are based on the following inequality in Eq. \eqref{local-4}
  \begin{equation}\label{number-machine-1}
  \big\|\bw_{R+1}-\bw_{R+1}^{(k)}\big\|_2\leq \sum_{r=1}^{R}\sum_{m=1}^{M}\sum_{t=1}^{K}\frac{\eta_{r,t}}{M}\mathfrak{C}_{m,r,t,k}\ibb_{[i_{m,r,t}=k]},
  \end{equation}
  where $\mathfrak{C}_{m,r,t,k}=\|\nabla f(\bw_{m,r,t};z_k)-\nabla f(\bw_{m,r,t}^{(k)};z_k')\|_2$, and $\ibb_{[i_{m,r,t}=k]}=1$ if $i_{m,r,t}=k$, and $0$ otherwise.
  Note the above upper bound is an average of $\xi_m:=\sum_{r=1}^{R}\sum_{t=1}^{K}\eta_{r,t}\mathfrak{C}_{m,r,t,k}\ibb_{[i_{m,r,t}=k]}$ over $m\in[M]$, which comes from the average scheme in local SGD.
  We take an expectation over both sides of Eq. \eqref{number-machine-1} to get $\ell_1$ on-average stability bounds.
  An average operator does not affect the expectation, which explains why the $\ell_1$ on-average stability bounds are independent of $M$. We take an expectation-variance decomposition to conduct  the $\ell_2$ stability analysis, and the resulting bound involves a term related to variance and a term related to expectation. The variance of an average of $M$ random variables decreases by a factor of $M$, which explains why the first term on the right-hand side of Eq. \eqref{local-l2} involves a factor of $1/M$. The second term in Eq. \eqref{local-l2} is independent of $M$ since the average does not affect expectation. This phenomenon also happens for minibatch SGD, where the average over a batch of size $b$ decreases the variance by a factor of $b$, and does not affect the expectation.

  In the following table, we summarize the comparison on the stability bounds of minibatch and local SGD for convex and smooth problems. Here $T$ is the number of iterations per machine, which is $R$ for minibatch SGD and $RK$ for local SGD. For simplicity, we ignore the discussion with optimistic bounds, and simply assume the empirical risks are bounded in expectation.

\begin{center}

\begin{tabular}{|c|c|c|}
  \hline
  Problems & $\ell_1$ on-average model stability & $\ell_2$ on-average model stability \\ \hline
   minibatch SGD & $\frac{T\eta}{n}$ & $\frac{\sqrt{T}\eta}{\sqrt{nb}}+\frac{T\eta}{n}$ \\ \hline
   local SGD & $\frac{T\eta}{n}$ & $\frac{\sqrt{T}\eta}{\sqrt{nM}}+\frac{T\eta}{n}$ \\
  \hline
\end{tabular}
\end{center}

Note that all the above bounds involve $\frac{T\eta}{n}$, which corresponds to an \emph{expectation} term in controlling the distance between two sequences of SGD iterates. We have either the term $\frac{\sqrt{T}\eta}{\sqrt{nb}}$ or $\frac{\sqrt{T}\eta}{\sqrt{nM}}$ for $\ell_2$ stability analysis, which corresponds to a \emph{variance}  and decreases as the batch size (number of machines) increases.

\end{remark}

We now use the above stability bounds to develop excess population risk bounds for local SGD. We first consider a convex case. The proof is given in Section \ref{sec:gen-local}. Note our stability analysis for local SGD is data-dependent in the sense of involving training errors. Our excess risk bounds  are not data-dependent since the existing optimization error bounds are not data-dependent~\citep{woodworth2020local}. It is interesting to develop data-dependent bounds for local SGD. 
\begin{theorem}[Risk Bound for Local SGD: Convex Case\label{thm:gen-local}]
Assume for all $z\in\zcal$, the map $\bw\mapsto f(\bw;z)$ is nonnegative, convex and $L$-smooth. Let $\{\bw_{m,r,t}\}$ be produced by the algorithm $A$ defined in \eqref{local-sgd} with $\eta_{r,t}=\eta\leq 2/L$.  Assume for all $r\in[R],t\in[K]$, $\ebb_{i_{m,r,t}}[\|\nabla f(\bw_{m,r,t};z_{i_{m,r,t}})-\nabla F_S(\bw_{m,r,t})\|_2^2]\leq\sigma^2$. Suppose we choose $\eta\asymp \|\bw^*\|_2\sqrt{n}/(KR\sqrt{L})$. If $KRM\asymp n$, $\eta\lesssim (K-1)^{-\frac{1}{2}}\|\bw^*\|_2^{\frac{1}{2}}/(nL)^{\frac{1}{4}}$ and $\eta\leq 1/(2L)$, then
$
\ebb[F(\bar{\bw}_{R,1})]-F(\bw^*)\lesssim \frac{\sqrt{L}\|\bw^*\|_2}{\sqrt{n}},
$
where
  $
  \bar{\bw}_{R,1}=\frac{1}{MKR}\sum_{m=1}^M\sum_{r=1}^{R}\sum_{t=1}^{K}\bw_{m,r,t}.
  $
\end{theorem}

\begin{remark}[Linear speedup\label{rem:gen-local}]
  Theorem \ref{thm:gen-local} shows that local SGD can achieve the minimax optimal excess population risk bounds $1/\sqrt{n}$ in the sense of matching the existing lower bounds~\citep{agarwal2009information}.
  We now discuss the speedup in the computation and we have $\eta\asymp\|\bw^*\|_2M/\sqrt{nL}$. Note $\eta\leq 2/L$ requires $M\lesssim \sqrt{nL}/\|\bw^*\|_2$. Furthermore, the condition $\eta\lesssim (K-1)^{-\frac{1}{2}}\|\bw^*\|_2^{\frac{1}{2}}/(nL)^{\frac{1}{4}}$ requires $M\lesssim (nL)^{\frac{1}{4}}/\sqrt{(K-1)\|\bw^*\|_2}$. Under these conditions, local SGD achieves a linear speedup in the sense that the computation per machine is of the order of $KR\asymp n/M$.

\end{remark}

Finally, we give risk bounds of local SGD for strongly convex problems to be proved in Section~\ref{sec:proof-risk-local-sg}.

\begin{theorem}[Risk Bounds for Local SGD: Strongly Convex Case\label{thm:risk-local-sg}]
Assume for all $z\in\zcal$, the map $\bw\mapsto f(\bw;z)$ is nonnegative, $\mu$-strongly convex and $L$-smooth. Let $\{\bw_{m,r,t}\}$ be produced by the algorithm $A$ defined in \eqref{local-sgd} with $\eta_{r,t}=\frac{4}{\mu(a+(r-1)K+t)}\leq2/L$ and $a\geq2L/\mu$.  Assume for all $r\in[R],t\in[K]$, $\ebb_{i_{m,r,t}}[\|\nabla f(\bw_{m,r,t};z_{i_{m,r,t}})-\nabla F_S(\bw_{m,r,t})\|_2^2]\leq\sigma^2$.
Assume $\sup_z\|\nabla f(A(S);z)\|_2\leq G$. If $KR\gtrsim \frac{n\sigma^2}{MG\sqrt{L}}$ and $\mu KR^2\gtrsim \frac{n\sqrt{L}}{G}$, then
\[
\ebb[F(\bar{\bw}_{R,2})]-F(\bw^*)\lesssim G\sqrt{L}\log(KR)/(n\mu),
\]
where
\[
S_R=\sum_{r=1}^{R}\sum_{t=1}^{K}(a+(r-1)K+t)\;\text{and}\;\bar{\bw}_{R,2}=\frac{1}{MS_R}\sum_{m=1}^{M}\sum_{r=1}^{R}\sum_{t=1}^{K}(a+(r-1)K+t)\bw_{m,r,t}.
\]
\end{theorem}
If $M\lesssim \frac{\sqrt{n\mu}\sigma^2}{\sqrt{G}L^{\frac{3}{4}}\sqrt{K}}$, we can choose $R\asymp \frac{n\sigma^2}{G\sqrt{L}KM}$ to show that $\mu KR^2\asymp \frac{\mu n^2\sigma^4}{G^2LKM^2}\gtrsim \frac{n\sqrt{L}}{G}$. Therefore, all the conditions of Theorem \ref{thm:risk-local-sg} hold, and we get the rate $G\sqrt{L}\log(KR)/(n\mu)$.
\begin{remark}[Comparison\label{rem:federated-existing}]
  Generalization bounds for agnostic federated learning were developed from a uniform convergence approach~\citep{mohri2019agnostic}, which involve Rademacher complexities of function spaces and are algorithm-independent. As a comparison, we study generalization from an algorithmic stability approach and get complexity-independent bounds.

  A federated stability was introduced to study the generalization of federated learning algorithms~\citep{chen2023minimax} in a strongly convex setting. As a comparison, our analysis also applies to general convex problems. Furthermore, their stability analysis was conducted for abstract approximate minimizers, while our stability analysis is developed for local SGD. Finally, their bound involves an upper bound of the loss function over a compact domain, and therefore cannot imply optimistic bounds.

   There is a recent work on the generalization of federated learning algorithms on a heterogeneous setup where the $i$-th local machine has its own dataset $S_i$~\citep{sun2023understanding}. For local SGD with a constant step size $\eta$, their generalization bounds are of the order of $O(n^{-1}RK\sigma\eta(1+K\eta))$ under a Lipschitz continuity assumption and a bounded variance assumption $\ebb[\|\nabla f(\bw;z_i)-\nabla F_{S_i}(\bw)\|^2_2]\leq\sigma^2$, where $z_i$ is drawn uniformly from $S_i$. While the bounds in~\citep{sun2023understanding} also involve $\|\nabla F(\bw_t)\|$, it is dominated by $\sigma$ and therefore cannot imply fast rates in an interpolation setting. As a comparison, our bounds in Eq. \eqref{local-l2} are optimistic and decay fast if $F_S(\bw_{m,r,t})$ decays to $0$. Furthermore, the analysis in~\citep{sun2023understanding} requires a Lipschitz condition on the loss function, which is removed in our analysis. Finally, we also develop $\ell_2$ on-average stability bounds, which are more challenging and illustrate the second-order information on the stability.
\end{remark}

\section{Proofs on Minibatch SGD\label{sec:proof-minibatch-sgd}}
\subsection{Proof of Theorem \ref{thm:on-average}\label{sec:proof-on-average}}
To prove Theorem \ref{thm:on-average}, we first introduce several lemmas. The following lemma shows the self-bounding property for nonnegative and smooth functions, meaning the magnitude of gradients can be bounded by function values~\citep{srebro2010smoothness,ying2017unregularized}.
\begin{lemma}[\citep{srebro2010smoothness}\label{lem:self-bound}]
  Assume for all $z$, the function $\bw\mapsto f(\bw;z)$ is nonnegative and $L$-smooth. Then $\|\nabla f(\bw;z)\|_2^2\leq2Lf(\bw;z)$.
\end{lemma}

In our analysis, we will use the concept of binomial distribution. Let $\mbox{Var}(X)$ denote the variance of a random variable $X$.
\begin{definition}[Binomial distribution]
The binomial distribution with parameters $n$ and $p$ is the discrete probability distribution of the number of successes in a sequence of $n$ independent trials, with the probability of success on a single trial denoted by $p$. We use $B(n,p)$ to denote the binomial distribution with parameters $n$ and $p$.
\end{definition}
\begin{lemma}\label{lem:binomial}
  If $X\sim B(n,p)$, then
  \[
  \ebb[X]=np\quad\text{and}\quad\mbox{Var}(X)=np(1-p).
  \]
\end{lemma}
A key property on establishing the stability of SGD is the non-expansiveness of the gradient-update operator established in the following lemma.
\begin{lemma}[\citep{hardt2016train}\label{lem:nonexpansive}]
Assume for all $z\in\zcal$, the function $\bw\mapsto f(\bw;z)$ is convex and $L$-smooth. Then for $\eta\leq2/L$ we know
  \begin{equation}\label{nonexpansive-a}
  \|\bw-\eta \nabla f(\bw;z)-\bw'+\eta \nabla f(\bw';z)\|_2\leq \|\bw-\bw'\|_2.
  \end{equation}
Furthermore, if $\bw\mapsto f(\bw;z)$ is $\mu$-strongly convex and $\eta\leq1/L$ then
\begin{gather}
\|\bw-\eta \nabla f(\bw;z)-\bw'+\eta \nabla f(\bw';z)\|_2\leq (1-\eta\mu/2)\|\bw-\bw'\|_2,\label{nonexpansive-b}\\
\|\bw-\eta \nabla f(\bw;z)-\bw'+\eta \nabla f(\bw';z)\|_2^2\leq (1-\eta\mu)\|\bw-\bw'\|_2^2.\label{nonexpansive-c}
\end{gather}
\end{lemma}
We are now ready to prove Theorem \ref{thm:on-average}. The analysis for $\ell_1$-stability bounds is standard~\citep{lei2020fine}. As a comparison, the analysis with the $\ell_2$-stability bounds requires new techniques such as the expectation-variance decomposition based on a representation of SGD with Binomial random variables.
For simplicity, we define $J_t=\{i_{t,1},\ldots,i_{t,b}\},t\in\nbb$.
\begin{proof}[Proof of Theorem \ref{thm:on-average}]
Define
\begin{equation}\label{alpha}
  \alpha_{t,m}=\big|\{j:i_{t,j}=m\}\big|,\quad\forall t\in\nbb,m\in[n],
\end{equation}
where we use $|S'|$  to denote the cardinality of a set $S'$.
That is, $\alpha_{t,m}$ is the number of indices equal to $m$ in the $t$-th iteration. Then the SGD update in Eq. \eqref{sgd} can be reformulated as
\begin{equation}\label{sgd-e}
\begin{split}
   \bw_{t+1}&=\bw_t-\frac{\eta_t}{b}\sum_{k=1}^{n}\alpha_{t,k}\nabla f(\bw_t;z_k),\\
   \bw_{t+1}^{(m)}&=\bw_t^{(m)}-\frac{\eta_t}{b}\sum_{k:k\neq m}\alpha_{t,k}\nabla f(\bw_t^{(m)};z_k)-\frac{\eta_t\alpha_{t,m}}{b}\nabla f(\bw_t^{(m)};z_m'),
\end{split}
\end{equation}
from which we know
\begin{multline}\label{stab-mini-1a}
  \|\bw_{t+1}-\bw_{t+1}^{(m)}\|_2 = \big\|\bw_{t}-\frac{\eta_t}{b}\sum_{k:k\neq m}\alpha_{t,k}\nabla f(\bw_t;z_k)-\frac{\eta_t\alpha_{t,m}}{b}\nabla f(\bw_t;z_m)\\
   - \bw^{(m)}_{t}+\frac{\eta_t}{b}\sum_{k:k\neq m}\alpha_{t,k}\nabla f(\bw^{(m)}_t;z_k)+\frac{\eta_t\alpha_{t,m}}{b}\nabla f(\bw^{(m)}_t;z_m') \big\|_2.
\end{multline}
For simplicity, introduce the notations for any $t\in[T]$ and $m\in[n]$
\begin{equation}\label{mini-11}
\Delta_{t,m}=\|\bw_{t}-\bw_{t}^{(m)}\|_2,\qquad \mathfrak{C}_{t,m}=\|\nabla f(\bw_t;z_m)-\nabla f(\bw^{(m)}_t;z_m')\|_2.
\end{equation}
Since $f$ is $L$-smooth and $\sum_{k:k\neq m}\alpha_{t,k}\leq b$, we know the function $\bw\mapsto \frac{1}{b}\sum_{k:k\neq m}\alpha_{t,k} f(\bw;z_k)$ is $L$-smooth.
By Lemma \ref{lem:nonexpansive} and the assumption $\eta_t\leq1/L$, we know
\begin{align*}
\Delta_{t+1,m} &\leq \big\|\bw_{t}-\frac{\eta_t}{b}\sum_{k:k\neq m}\alpha_{t,k}\nabla f(\bw_t;z_k)
   - \bw^{(m)}_{t}+\frac{\eta_t}{b}\sum_{k:k\neq m}\alpha_{t,k}\nabla f(\bw^{(m)}_t;z_k)\big\|_2+\frac{\eta_t\alpha_{t,m}\mathfrak{C}_{t,m}}{b}\\
&\leq  \Delta_{t,m}+\frac{\eta_t\alpha_{t,m}\mathfrak{C}_{t,m}}{b}.
\end{align*}
We can apply the above inequality recursively and derive (note $\bw_1=\bw_1^{(m)}$)
\begin{equation}\label{mini-4}
\Delta_{t+1,m}\leq \frac{1}{b}\sum_{k=1}^{t}\eta_k\alpha_{k,m}\mathfrak{C}_{k,m}.
\end{equation}
According to the definition of $\alpha_{t,k}$, we know that $\alpha_{t,k}$ is a random variable following from the binomial distribution $B(b,1/n)$ with parameters $b$ and $1/n$, from which we know
\begin{equation}\label{mini-6}
\ebb[\alpha_{t,k}]=b/n,\quad \var(\alpha_{t,k})=b(1-1/n)\cdot(1/n)\leq b/n.
\end{equation}
Furthermore, Lemma \ref{lem:self-bound} implies
\begin{equation}\label{mini-5}
\mathfrak{C}_{k,m}\leq \|\nabla f(\bw_k;z_m)\|_2+\|\nabla f(\bw_k^{(m)};z_m')\|_2
\leq \sqrt{2Lf(\bw_k;z_m)}+\sqrt{2Lf(\bw_k^{(m)};z_m')}.
\end{equation}
We can combine the above inequality, Eq. \eqref{mini-6} and Eq. \eqref{mini-4} together to derive
\begin{align}
  \ebb[\Delta_{t+1,m}] & \leq \frac{1}{b}\sum_{k=1}^{t}\eta_k\ebb\big[\alpha_{k,m}\mathfrak{C}_{k,m}\big] = \frac{1}{b}\sum_{k=1}^{t}\eta_k\ebb\big[\ebb_{J_k}[\alpha_{k,m}]\mathfrak{C}_{k,m}\big] = \frac{1}{n}\sum_{k=1}^{t}\eta_k\ebb\big[\mathfrak{C}_{k,m}\big],
\end{align}
where we have used the fact that $\mathfrak{C}_{k,m}$ is independent of $J_k$. According to the symmetry between $z_m$ and $z_m'$, we know $\ebb[f(\bw_t^{(m)};z_m')]=\ebb[f(\bw_t;z_m)]$ and therefore Eq. \eqref{mini-5} implies $\ebb\big[\mathfrak{C}_{k,m}\big]\leq 2\sqrt{2L}\ebb\big[\sqrt{f(\bw_k;z_m)}\big]$. It then follows that
\[
\ebb[\Delta_{t+1,m}]\leq \frac{2\sqrt{2L}}{n}\sum_{k=1}^{t}\eta_k\ebb\big[\sqrt{f(\bw_k;z_m)}\big].
\]
It then follows from the concavity of $x\mapsto \sqrt{x}$ that
\begin{align}
  \frac{1}{n}\sum_{m=1}^{n}\ebb[\|\bw_{t+1}-\bw_{t+1}^{(m)}\|_2] & \leq \frac{2}{n}\sum_{m=1}^{n}\sum_{k=1}^{t}\frac{\eta_k}{n}\ebb\big[\sqrt{2Lf(\bw_k;z_m)}\big]\notag \\
   & \leq \sum_{k=1}^{t}\frac{2\eta_k}{n}\sqrt{\frac{2L }{n}\sum_{m=1}^{n}\ebb[f(\bw_k;z_m)]}
   = \sum_{k=1}^{t}\frac{2\eta_k\sqrt{2L\ebb[F_S(\bw_k)]}}{n}.\label{sum-concave}
\end{align}
This establishes the stated bound \eqref{on-average-a}.

We now prove Eq. \eqref{on-average-b}. We introduce an expectation-variance decomposition in \eqref{mini-4} as follows
\[
  \Delta_{t+1,m}\leq \frac{1}{b}\sum_{k=1}^{t}\eta_k\big(\alpha_{k,m}-b/n\big)\mathfrak{C}_{k,m}+\frac{1}{n}\sum_{k=1}^{t}\eta_k\mathfrak{C}_{k,m}.
\]
We take square on both sides followed with an expectation (w.r.t. $S$ and $J_1,\ldots,J_t$) and use $(a+b)^2\leq 2(a^2+b^2)$ to show
\begin{align*}
\ebb\big[\Delta_{t+1,m}^2\big] &\leq \frac{2}{b^2}\ebb\Big[\Big(\sum_{k=1}^{t}\eta_k\big(\alpha_{k,m}-b/n\big)\mathfrak{C}_{k,m}\Big)^2\Big]+\frac{2}{n^2}\ebb\Big[\Big(\sum_{k=1}^{t}\eta_k\mathfrak{C}_{k,m}\Big)^2\Big]\\
& = \frac{2}{b^2}\ebb\Big[\sum_{k,k'=1}^{t}\eta_k\eta_{k'}\big(\alpha_{k,m}-b/n\big)\big(\alpha_{k',m}-b/n\big)\mathfrak{C}_{k,m}\mathfrak{C}_{k',m}\Big]+\frac{2}{n^2}\ebb\Big[\Big(\sum_{k=1}^{t}\eta_k\mathfrak{C}_{k,m}\Big)^2\Big].
\end{align*}
For any $k\neq k'$, it follows from $\ebb_{J_{k'}}[\alpha_{k',m}]=b/n$ (we can assume $k<k'$ without loss of generality)
\begin{align}
  \ebb\Big[\big(\alpha_{k,m}-b/n\big)\big(\alpha_{k',m}-b/n\big)\mathfrak{C}_{k,m}\mathfrak{C}_{k',m}\Big] & = \ebb\ebb_{J_{k'}}\Big[\big(\alpha_{k,m}-b/n\big)\big(\alpha_{k',m}-b/n\big)\mathfrak{C}_{k,m}\mathfrak{C}_{k',m}\Big]\notag\\
   & = \ebb\Big[\big(\alpha_{k,m}-b/n\big)\ebb_{J_{k'}}\big[\alpha_{k',m}-b/n\big]\mathfrak{C}_{k,m}\mathfrak{C}_{k',m}\Big]=0,\label{mini-14}
\end{align}
where we have used the fact that $\alpha_{k,m},\mathfrak{C}_{k,m}$ and $\mathfrak{C}_{k',m}$ are independent of $J_{k'}$. It then follows from Eq. \eqref{mini-6} that
\begin{align*}
  \ebb\big[\Delta_{t+1,m}^2\big]  & \leq \frac{2}{b^2}\ebb\Big[\sum_{k=1}^{t}\eta_k^2\big(\alpha_{k,m}-b/n\big)^2\mathfrak{C}_{k,m}^2\Big]+\frac{2}{n^2}\ebb\Big[\Big(\sum_{k=1}^{t}\eta_k\mathfrak{C}_{k,m}\Big)^2\Big]\\
  & = \frac{2}{b^2}\ebb\Big[\sum_{k=1}^{t}\eta_k^2\var(\alpha_{k,m})\mathfrak{C}_{k,m}^2\Big]+\frac{2}{n^2}\ebb\Big[\Big(\sum_{k=1}^{t}\eta_k\mathfrak{C}_{k,m}\Big)^2\Big]\\
  &\leq \frac{2}{nb}\ebb\Big[\sum_{k=1}^{t}\eta_k^2\mathfrak{C}_{k,m}^2\Big]+\frac{2}{n^2}\ebb\Big[\Big(\sum_{k=1}^{t}\eta_k\mathfrak{C}_{k,m}\Big)^2\Big]\\
  & \leq \frac{2}{nb}\ebb\Big[\sum_{k=1}^{t}\eta_k^2\mathfrak{C}_{k,m}^2\Big]+\frac{8}{n^2}\ebb\Big[\Big(\sum_{k=1}^{t}\eta_k\|\nabla f(\bw_k;z_m)\|_2\Big)^2\Big],
\end{align*}
where we have used the following inequality in the last step
\begin{align}
\ebb\Big[\Big(\sum_{k=1}^{t}\eta_k\mathfrak{C}_{k,m}\Big)^2\Big]
&\leq 2\ebb\Big[\Big(\sum_{k=1}^{t}\eta_k\|\nabla f(\bw_k;z_m)\|_2\Big)^2\Big]+
2\ebb\Big[\Big(\sum_{k=1}^{t}\eta_k\|\nabla f(\bw_k^{(m)};z_m')\|_2\Big)^2\Big]\notag\\
& = 4\ebb\Big[\Big(\sum_{k=1}^{t}\eta_k\|\nabla f(\bw_k;z_m)\|_2\Big)^2\Big].\label{mini-13}
\end{align}
Analogous to Eq. \eqref{mini-5}, we have
\begin{align}
\ebb[\mathfrak{C}_{k,m}^2] &\leq 2\ebb[\|\nabla f(\bw_k;z_m)\|_2^2]+2\ebb[\|\nabla f(\bw_k^{(m)};z_m')\|_2^2]\notag\\
&\leq 4L\ebb\big[f(\bw_k;z_m)+f(\bw_k^{(m)};z_m')\big]=8L\ebb\big[f(\bw_k;z_m)].\label{mini-10}
\end{align}
It then follows that
\[
\ebb\big[\Delta_{t+1,m}^2\big]
  \leq \frac{16L}{nb}\sum_{k=1}^{t}\eta_k^2\ebb\big[f(\bw_k;z_m)]+\frac{8}{n^2}\ebb\Big[\Big(\sum_{k=1}^{t}\eta_k\|\nabla f(\bw_k;z_m)\|_2\Big)^2\Big].
\]
We take an average over all $m\in[n]$ and get
\begin{align*}
\frac{1}{n}\sum_{m=1}^{n}\ebb\big[\Delta_{t+1,m}^2\big]
  &\leq \frac{16L}{n^2b}\sum_{k=1}^{t}\sum_{m=1}^{n}\eta_k^2\ebb\big[f(\bw_k;z_m)]+\frac{8}{n^3}\sum_{m=1}^{n}\ebb\Big[\Big(\sum_{k=1}^{t}\eta_k\|\nabla f(\bw_k;z_m)\|_2\Big)^2\Big]\\
  & = \frac{16L}{nb}\sum_{k=1}^{t}\eta_k^2\ebb\big[F_S(\bw_k)]+\frac{8}{n^3}\sum_{m=1}^{n}\ebb\Big[\Big(\sum_{k=1}^{t}\eta_k\|\nabla f(\bw_k;z_m)\|_2\Big)^2\Big].
\end{align*}
The proof is completed.
\end{proof}

\begin{remark}[Novelty in the analysis\label{rem:novelty}]
Similar stability bounds involving $F_S(\bw_k)$ were developed for the vanilla SGD~\citep{lei2020fine}. Their argument needs to distinguish two cases according to whether the algorithm chooses a particular example at each iteration. This argument does not work for the minibatch SGD since we draw $b$ examples per iteration and we can draw the particular example several times. We bypass this difficulty by introducing the \emph{expectation-variance decomposition} and self-bounding property~\citep{kuzborskij2018data,lei2020fine} into the stability analysis based on a reformulation of minibatch SGD with binomial variables.
Indeed, the paper \citep{lei2020fine} considers SGD with $\tilde{\bw}_{t+1}=\tilde{\bw}_t-\eta_t\nabla f(\tilde{\bw}_t;z_{i_t})$. Their discussion controls $\|\tilde{\bw}_{t+1}-\tilde{\bw}^{(m)}_{t+1}\|_2^2$ by considering two cases: $i_t=m$ or $i_t\neq m$. If $i_t=m$, they use
\begin{equation}\label{lei-trick}
\|\bv_1+\bv_2\|_2^2\leq (1+p)\|\bv_1\|_2^2+(1+1/p)\|\bv_2\|_2^2
\end{equation}
and get $\|\tilde{\bw}_{t+1}\!-\!\tilde{\bw}^{(m)}_{t+1}\|_2^2\!\leq\! (1\!+\!p)\|\tilde{\bw}_{t}\!-\!\tilde{\bw}^{(m)}_{t}\|_2^2\!+\!(1\!+\!1/p)\eta_t^2\mathfrak{C}_{t,m}^2$. Since $i_t=m$ happens with probability $1/n$, they derive
\begin{equation}\label{novel-1}
\ebb[\|\tilde{\bw}_{t+1}-\tilde{\bw}^{(m)}_{t+1}\|_2^2]\leq (1+{p/n})\ebb[\|\tilde{\bw}_{t}-\tilde{\bw}^{(m)}_{t}\|_2^2]+O\Big(\big(\frac{1}{n}+\frac{1}{np}\big)\eta_t^2\ebb[\mathfrak{C}_{t,m}^2]\Big).
\end{equation}
For minibatch SGD, we may select $i_m$ several times and cannot divide the discussions into two cases as in~\citep{lei2020fine}. Instead, we reformulate SGD as Eq. \eqref{sgd-e} with $\alpha_{t,k}$ being a binomial random variable. Furthermore, even with the formulation, the existing techniques \citep{lei2020fine} would imply suboptimal bounds. Indeed, applying \eqref{lei-trick} to Eq. \eqref{stab-mini-1a} would imply
\begin{align}
  \ebb[\Delta_{t+1,m}^2]&\leq (1+p)\ebb[\Delta_{t,m}^2]+\eta_t^2b^{-2}(1+1/p)\ebb[\alpha_{t,m}^2\mathfrak{C}_{t,m}^2]\notag\\
  &\leq  (1+p)\ebb[\Delta_{t,m}^2]+2\eta_t^2b^{-1}n^{-1}(1+1/p)\ebb[\mathfrak{C}_{t,m}^2],\label{lei-1}
\end{align}
where we have used $\ebb_{J_t}[\alpha_{t,m}^2]\leq 2b/n$.
The key difference is we have a factor of $1+p/n$ for SGD and $1+p$ for minibatch SGD. To see how Eq. \eqref{lei-1} implies sub-optimal bounds, we continue the deduction as follows. We apply Eq. \eqref{lei-1} recursively and get
\begin{align*}
  \ebb[\Delta_{t+1,m}^2]
  & \leq 2b^{-1}n^{-1}(1+1/p)\sum_{k=1}^{t}(1+p)^{t+1-k}\eta_k^2\ebb[\mathfrak{C}_{k,m}^2] \leq 2b^{-1}n^{-1}(1+1/p)(1+p)^{t}\sum_{k=1}^{t}\eta_k^2\ebb[\mathfrak{C}_{k,m}^2] \\
   & \leq 2b^{-1}n^{-1}(1+t)e\sum_{k=1}^{t}\eta_k^2\ebb[\mathfrak{C}_{k,m}^2] \leq16Lb^{-1}n^{-1}(1+t)e\sum_{k=1}^{t}\eta_k^2\ebb[f(\bw_k;z_m)],
\end{align*}
where we choose $p=1/t$ and use $(1+1/t)^t\leq e$ in the last second inequality, and use Eq. \eqref{mini-10} in the last inequality. An average over all $m\in[n]$ implies
\begin{equation}\label{lei-2}
\frac{1}{n}\sum_{m=1}^{n}\ebb[\Delta_{t+1,m}^2]\leq \frac{16L(1+t)e\eta^2}{nb}\sum_{k=1}^{t}\ebb[F_S(\bw_k)],
\end{equation}
which is much worse than Eq. \eqref{on-average-b}. Indeed, if $\ebb[F_S(\bw_k)]\lesssim 1$, then Eq. \eqref{lei-2} implies $\frac{1}{n}\sum_{m=1}^{n}\ebb[\Delta_{t+1,m}^2]\lesssim t^2\eta^2/(nb)$. As a comparison, Eq. \eqref{on-average-b} implies $\frac{1}{n}\sum_{m=1}^{n}\ebb[\Delta_{t+1,m}^2]\lesssim t\eta^2/(nb)+t^2\eta^2/n^2$. Note $t\eta^2/(nb)$ outperforms $t^2\eta^2/(nb)$ by a factor of $t$, and $t^2\eta^2/n^2$ outperforms $t^2\eta^2/(nb)$ by a factor of $n/b$.

We significantly improve the analysis in~\citep{lei2020fine} by introducing new techniques in the analysis with $\ell_2$ on-average model stability. Our idea is to use an expectation-variance decomposition
$ \Delta_{t+1,m}\leq \frac{1}{b}\sum_{k=1}^{t}\eta_k\big(\alpha_{k,m}-b/n\big)\mathfrak{C}_{k,m}+\frac{1}{n}\sum_{k=1}^{t}\eta_k\mathfrak{C}_{k,m}$. The key observation is that
$\ebb\big[\big(\alpha_{k,m}-b/n\big)\mathfrak{C}_{k,m}\big(\alpha_{k',m}-b/n\big)\mathfrak{C}_{k',m}\big]=0$ if $k\neq k'$. This removes the cross-over terms when taking a square followed by an expectation, and implies
\[
\ebb\big[\Delta_{t+1,m}^2\big]  \leq \frac{2}{b^2}\ebb\Big[\sum_{k=1}^{t}\eta_k^2\big(\alpha_{k,m}-b/n\big)^2\mathfrak{C}_{k,m}^2\Big]+\frac{2}{n^2}\ebb\Big[\Big(\sum_{k=1}^{t}\eta_k\mathfrak{C}_{k,m}\Big)^2\Big].
\]
It is also possible to derive high-order stability bounds under a Lipschitzness assumption. We omit the discussions for simplicity.

\end{remark}

\begin{remark}[Lower bounds\label{rem:lower}]
  Recently, lower bounds on the uniform stability have also received increasing attention. Let $\unif$ be the uniform stability of SGD with $t$ iterations and a constant step size $\eta$.
  For nonsmooth and Lipschitz loss functions, it was shown $\unif\gtrsim\min\{1,t/n\}\eta\sqrt{t}+\eta t/n$ for convex problems~\citep{bassily2020stability}, $\unif\gtrsim1/\mu\sqrt{n}$ for $\mu$-strongly convex problems ($\mu\geq1/\sqrt{n}$) and $\unif\gtrsim\eta^2n$ for nonconvex problems ($\eta\leq1/\sqrt{n}$)~\citep{koren2022benign}.
  For smooth loss functions, it was shown $\unif\gtrsim\eta t/n$ for convex and Lipschitz problems,  and $\unif\gtrsim1/(\mu n)$ for $\mu$-strongly convex problems~\citep{zhang2022stability}. It is clear that our on-average stability bounds in Eq. \eqref{on-average-a} match the existing lower bounds on uniform stability in the convex and smooth case. 
\end{remark}


Finally, we give some direct corollaries of Theorem \ref{thm:on-average}.
By the Cauchy-Schwarz's inequality $(\sum_{k=1}^{t}a_k)^2\leq t\sum_{k=1}^{t}a_k^2$, Eq. \eqref{on-average-b} further implies
\begin{align}
  \frac{1}{n}\sum_{m=1}^{n}\ebb\big[\|\bw_{t+1}\!-\!\bw_{t+1}^{(m)}\|_2^2\big]
  &\!\leq\! \frac{16L}{nb}\sum_{k=1}^{t}\eta_k^2\ebb\big[F_S(\bw_k)]\!+\!\frac{8t}{n^3}\sum_{m=1}^{n}\sum_{k=1}^{t}\eta_k^2\ebb\big[\|\nabla f(\bw_k;z_m)\|_2^2\big]\notag\\
  & \leq \Big(\frac{16L}{nb}\!+\!\frac{16Lt}{n^2}\Big)\sum_{k=1}^{t}\eta_k^2\ebb\big[F_S(\bw_k)],\label{on-average-c}
\end{align}
where we use $\|\nabla f(\bw_k;z_m)\|_2^2\leq 2Lf(\bw_k;z_m)$ due to the self-bounding property.
If $b=1$, our analysis implies stability bounds of order $L\big(\frac{1}{n}+\frac{t}{n^2}\big)\sum_{k=1}^{t}\eta_k^2\ebb\big[F_S(\bw_k)]$, which match the stability bounds for SGD~\citep{lei2020fine}.
Furthermore, under a stronger self-bounding property $\|\nabla f(\bw_k;z_m)\|_2\leq f(\bw_k;z_m)$ (e.g., logistic loss)~\citep{schliserman2022stability}, Eq. \eqref{on-average-b} implies
\begin{align}
\frac{1}{n}\sum_{m=1}^{n}\ebb\big[\|\bw_{t+1}-\bw_{t+1}^{(m)}\|_2^2\big]
  & \leq \frac{16L}{nb}\sum_{k=1}^{t}\eta_k^2\ebb\big[F_S(\bw_k)]+\frac{8}{n^3}\sum_{m=1}^{n}\ebb\Big[\Big(\sum_{k=1}^{t}\eta_kf(\bw_k;z_m)\Big)^2\Big]. \label{on-average-d}
\end{align}

\subsection{Proof of Theorem \ref{thm:risk-mini}\label{sec:proof-risk-mini}}
In this section, we present the proof of Theorem \ref{thm:risk-mini} on excess population risk bounds of minibatch SGD. We first introduce the following optimization error bounds.
\begin{lemma}[Optimization Errors of Minibatch SGD: Convex Case\label{lem:opt-cotter}]
  Assume for all $z\in\zcal$, the map $\bw\mapsto f(\bw;z)$ is nonnegative, convex and $L$-smooth.
  Let $\{\bw_t\}$ be produced by Eq. \eqref{sgd} with $\eta\leq1/L$. Then the following inequality holds for all $\bw$
  \begin{equation}\label{opt-cotter}
  \frac{1}{R}\sum_{t=1}^{R}\ebb_A\big[F_S(\bw_t)\big]-F_S(\bw)\leq \frac{2\eta L}{bR}\sum_{t=1}^{R}\ebb_A[F_S(\bw_t)]+\frac{\|\bw\|_2^2}{2\eta R}+\frac{F_S(\bw_1)}{R}.
  \end{equation}
\end{lemma}
\begin{proof}
Denote $B_t=\{z_{i_{t,1}},\ldots,z_{i_{t,b}}\}$ and $f(\bw;B_t)=\frac{1}{b}\sum_{j=1}^bf(\bw;z_{i_{t,j}})$. Then the update of minibatch SGD can be written as
\[
\bw_{t+1}=\bw_t-\eta\nabla f(\bw_t;B_t).
\]
Since $\ebb_{B_t}[f(\bw_t;B_t)]=F_S(\bw_t)$ we know
\begin{align}
  \ebb_A[\|\nabla f(\bw_t;B_t)\|_2^2] & = \ebb_A[\|\nabla f(\bw_t;B_t)-\nabla F_S(\bw_t)\|_2^2]+\ebb_A[\|\nabla F_S(\bw_t)\|_2^2]\notag\\
  & =\frac{1}{b}\ebb_A[\|\nabla f(\bw_t;z_{i_{t,1}})-\nabla F_S(\bw_t)\|_2^2] +\ebb_A[\|\nabla F_S(\bw_t)\|_2^2]\notag\\
  & = \frac{\ebb_A[\|\nabla f(\bw_t;z_{i_{t,1}})\|_2^2]}{b} - \frac{\ebb_A[\|\nabla F_S(\bw_t)\|_2^2]}{b} + \ebb_A[\|\nabla F_S(\bw_t)\|_2^2]\notag\\
  & \leq \frac{2L\ebb_A[f(\bw_t;z_{i_{t,1}})]}{b} - \frac{\ebb_A[\|\nabla F_S(\bw_t)\|_2^2]}{b} + \ebb_A[\|\nabla F_S(\bw_t)\|_2^2]\notag\\
  & \leq \frac{2L\ebb_A[F_S(\bw_t)]}{b}+ \ebb_A[\|\nabla F_S(\bw_t)\|_2^2],\label{minibatch-opt-0}
\end{align}
where we have used the self-bounding property of smooth functions. Furthermore, by the convexity of $f$ we know
\begin{align*}
  \|\bw_{t+1}-\bw\|_2^2 & = \|\bw_t-\bw\|_2^2 + \eta^2\|\nabla f(\bw_t;B_t)\|_2^2 + 2\eta\langle\bw-\bw_t,\nabla f(\bw_t;B_t)\rangle \\
   & \leq \|\bw_t-\bw\|_2^2 + \eta^2\|\nabla f(\bw_t;B_t)\|_2^2 + 2\eta(f(\bw;B_t)-f(\bw_t;B_t)).
\end{align*}
It then follows that
\[
\ebb_A[\|\bw_{t+1}-\bw\|_2^2]\leq
\ebb_A[\|\bw_t-\bw\|_2^2] + \frac{2L\eta^2\ebb_A[F_S(\bw_t)]}{b} + \eta^2\ebb_A[\|\nabla F_S(\bw_t)\|_2^2] + 2\eta\ebb_A[F_S(\bw)-F_S(\bw_t)].
\]
Taking a summation of the above inequality gives ($\bw_1=0$)
\begin{equation}\label{minibatch-opt-2}
2\eta\sum_{t=1}^{R}\ebb_A[F_S(\bw_t)-F_S(\bw)]\leq \|\bw\|_2^2 + \frac{2L\eta^2}{b}\sum_{t=1}^{R}\ebb_A[F_S(\bw_t)]+\eta^2\sum_{t=1}^{R}\ebb_A[\|\nabla F_S(\bw_t)\|_2^2].
\end{equation}
By the $L$-smoothness of $F_S$ and Eq. \eqref{minibatch-opt-0} we have
\begin{align*}
  \ebb_A[F_S(\bw_{t+1})] & \leq \ebb_A[F_S(\bw_t)] + \ebb_A[\langle\nabla F_S(\bw_t),\bw_{t+1}-\bw_t\rangle] + \frac{L\ebb_A[\|\bw_{t+1}-\bw_t\|_2^2]}{2} \\
  & = \ebb_A[F_S(\bw_t)] -\eta \ebb_A[\langle\nabla F_S(\bw_t),\nabla f(\bw_t;B_t)\rangle] + \frac{L\eta^2\ebb_A[\|\nabla f(\bw_t;B_t)\|_2^2]}{2}\\
  & \leq \ebb_A[F_S(\bw_t)] -\eta\ebb_A[\|\nabla F_S(\bw_t)\|_2^2] + \frac{L^2\eta^2\ebb_A[F_S(\bw_t)]}{b}+ \frac{L\eta^2\ebb_A[\|\nabla F_S(\bw_t)\|_2^2]}{2}.
\end{align*}
It then follows from $\eta\leq 1/L$ that
\[
\frac{\eta}{2}\sum_{t=1}^{R}\ebb_A[\|\nabla F_S(\bw_t)\|_2^2]\leq \ebb_A[F_S(\bw_1)]+\frac{L^2\eta^2\sum_{t=1}^{R}\ebb_A[F_S(\bw_t)]}{b}.
\]
We  combine the above inequality and Eq. \eqref{minibatch-opt-2} to derive (note $\eta\leq1/L$)
\begin{align*}
2\eta\sum_{t=1}^{R}\ebb_A[F_S(\bw_t)-F_S(\bw)] &\leq \|\bw\|_2^2 + \frac{2L\eta^2}{b}\sum_{t=1}^{R}\ebb_A[F_S(\bw_t)]+2\eta F_S(\bw_1)+\frac{2L^2\eta^3\sum_{t=1}^{R}\ebb_A[F_S(\bw_t)]}{b}\\
& \leq \|\bw\|_2^2 + \frac{4L\eta^2}{b}\sum_{t=1}^{R}\ebb_A[F_S(\bw_t)]+2\eta F_S(\bw_1).
\end{align*}
The proof is completed.
\end{proof}
\begin{proof}[Proof of Theorem \ref{thm:risk-mini}]
  We choose $\bw=\bw^*$ and take expectations {w.r.t. $S$} over both sides of Eq. \eqref{opt-cotter} to get
  \begin{equation}\label{mini-1}
    \frac{1}{R}\sum_{t=1}^{R}\ebb\big[F_S(\bw_t)\big]-F(\bw^*)\leq \frac{2\eta L}{bR}\sum_{t=1}^{R}\ebb[F_S(\bw_t)]+\frac{\|\bw^*\|_2^2}{2\eta R}+\frac{F(\bw_1)}{R}.
  \end{equation}
  We consider two cases. If $\frac{1}{R}\sum_{t=1}^{R}\ebb[F_S(\bw_t)]\leq F(\bw^*)$, then this means that the optimization error in Eq. \eqref{decomposition} is non-positive (this is the easier case since one does not need to consider the optimization error)
  \[
  \ebb[F_S(\bar{\bw}_R)]\leq \frac{1}{R}\sum_{t=1}^{R}\ebb[F_S(\bw_t)]\leq F(\bw^*)=\ebb[F_S(\bw^*)].
  \]
  We now consider the case $\frac{1}{R}\sum_{t=1}^{R}\ebb[F_S(\bw_t)]\geq F(\bw^*)$. Then it follows from Eq. \eqref{mini-1} that
  \begin{align*}
    \frac{1}{R}\sum_{t=1}^{R}\ebb\big[F_S(\bw_t)\big]-F(\bw^*) &\leq \frac{2\eta L}{bR}\sum_{t=1}^{R}\ebb[F_S(\bw_t)-F(\bw^*)]+\frac{2\eta L}{bR}\sum_{t=1}^{R}F(\bw^*)+\frac{\|\bw^*\|_2^2}{2\eta R}+\frac{F(\bw_1)}{R}\\
    & \leq \frac{1}{2R}\sum_{t=1}^{R}\ebb[F_S(\bw_t)-F(\bw^*)]+\frac{2\eta L}{bR}\sum_{t=1}^{R}F(\bw^*)+\frac{\|\bw^*\|_2^2}{2\eta R}+\frac{F(\bw_1)}{R},
  \end{align*}
  where we have used $\eta\leq b/(4L)$ due to $b\geq 2$.
  It then follows that
  \begin{equation}\label{mini-2}
  \frac{1}{R}\sum_{t=1}^{R}\ebb\big[F_S(\bw_t)\big]-F(\bw^*) \leq \frac{4\eta LF(\bw^*)}{b}+\frac{\|\bw^*\|_2^2}{\eta R}+\frac{2F(\bw_1)}{R}.
  \end{equation}
  By Lemma \ref{lem:gen-model-stab} (Part (b)) and Eq. \eqref{on-average-c}, we know
  \[
    \ebb[F(\bar{\bw}_{R})-F_S(\bar{\bw}_{R})] \leq \frac{L}{\gamma}\ebb[F_S(\bar{\bw}_{R})]+
    (L+\gamma)\Big(\frac{8L}{nb}+\frac{8LR}{n^2}\Big)\sum_{t=1}^{R}\eta_t^2\ebb\big[F_S(\bw_t)].
  \]
  Eq. \eqref{mini-2} implies that
  \[
  \frac{1}{R}\sum_{t=1}^{R}\ebb\big[F_S(\bw_t)\big]\lesssim F(\bw^*)+\|\bw^*\|_2^2/(\eta R).
  \]
  We combine the above two inequalities together and derive (note $F_S(\bar{\bw}_R)\leq\frac{1}{R}\sum_{t=1}^{R}F_S(\bw_t)$)
  \[
  \ebb[F(\bar{\bw}_{R})-F_S(\bar{\bw}_{R})]\lesssim L\Big(\frac{F(\bw^*)+\|\bw^*\|_2^2/(\eta R)}{\gamma}\Big)+
    L(L+\gamma)\eta^2\Big(\frac{1}{nb}+\frac{R}{n^2}\Big)\big(RF(\bw^*)+\|\bw^*\|_2^2/\eta\big).
  \]
  We plug the above generalization error bounds, the optimization error bounds in Eq. \eqref{mini-2} back into Eq. \eqref{decomposition}, and derive
  \begin{multline*}
    \ebb[F(\bar{\bw}_{R})]-F(\bw^*)\lesssim \frac{\eta LF(\bw^*)}{b}+\frac{\|\bw^*\|_2^2}{\eta R}+\frac{LF(\bw^*)+L\|\bw^*\|_2^2/(\eta R)}{\gamma}\\
    +
    L(L+\gamma)\eta^2\Big(\frac{1}{nb}+\frac{R}{n^2}\Big)\big(RF(\bw^*)+\|\bw^*\|_2^2/\eta\big).
  \end{multline*}
  The proof is completed.
\end{proof}

\begin{proof}[Proof of Corollary \ref{cor:mini}]
  We first consider the case that $F(\bw^*)\geq 4Lb^2\|\bw^*\|_2^2/n$. In this case, we have $\frac{\|\bw^*\|_2b}{\sqrt{LnF(\bw^*)}}\leq \frac{1}{2L}$ and therefore $\eta=\frac{\|\bw^*\|_2b}{\sqrt{LnF(\bw^*)}}$. We have
  \begin{equation}\label{mini-20}
    \eta R\asymp \frac{\|\bw^*\|_2b}{\sqrt{LnF(\bw^*)}}\frac{n}{b}=\frac{\sqrt{n}\|\bw^*\|_2}{\sqrt{LF(\bw^*)}}
  \end{equation}
  and therefore
  \[
  F(\bw^*)\eta R\asymp \frac{\sqrt{nF(\bw^*)}\|\bw^*\|_2}{\sqrt{L}}\geq \frac{2\sqrt{L}b\|\bw^*\|_2^2}{\sqrt{L}}=2b\|\bw^*\|_2^2.
  \]
  Theorem \ref{thm:risk-mini} together with $R\asymp n/b$ then implies
  \[
    \ebb[F(\bar{\bw}_{R})]\!-\!F(\bw^*)\lesssim \frac{\eta LF(\bw^*)}{b}\!+\!\frac{\|\bw^*\|_2^2}{\eta R}
    +
    LF(\bw^*)\bigg(\frac{1}{\gamma}+(L\!+\!\gamma)\eta^2\frac{R^2}{n^2}\bigg).
  \]
  Since $\eta=\frac{\|\bw^*\|_2b}{\sqrt{nLF(\bw^*)}}$, $R\asymp\frac{n}{b}$ and $\gamma=\sqrt{LnF(\bw^*)}/\|\bw^*\|_2$, we know
  \begin{gather*}
    \frac{\eta L F(\bw^*)}{b}\asymp\frac{Lb\|\bw^*\|}{\sqrt{LnF(\bw^*)}}\frac{F(\bw^*)}{b}\asymp\frac{\|\bw^*\|_2(LF(\bw^*))^{\frac{1}{2}}}{\sqrt{n}}, \\
     \frac{LF(\bw^*)}{\gamma}\asymp\frac{LF(\bw^*)\|\bw^*\|_2}{\sqrt{LnF(\bw^*)}}\asymp\frac{(LF(\bw^*))^{\frac{1}{2}}\|\bw^*\|_2}{\sqrt{n}}
  \end{gather*}
  and
  \begin{align*}
     & \frac{L(L+\gamma)\eta^2R^2F(\bw^*)}{n^2}\asymp \frac{L(L+(LnF(\bw^*))^{\frac{1}{2}}/\|\bw^*\|_2)\|\bw^*\|_2^2b^2R^2F(\bw^*)}{n^2LnF(\bw^*)}\\
     &\asymp \frac{(L+(LnF(\bw^*))^{\frac{1}{2}}/\|\bw^*\|_2)\|\bw^*\|_2^2}{n}\lesssim \frac{(LF(\bw^*))^{\frac{1}{2}}\|\bw^*\|_2}{\sqrt{n}}.
  \end{align*}
  We  plug the above inequalities back into Eq. \eqref{mini-20} and get $\ebb[F(\bar{\bw}_{R})]-F(\bw^*)\lesssim \frac{(LF(\bw^*))^{\frac{1}{2}}\|\bw^*\|_2}{\sqrt{n}}$.

  We now consider the case $F(\bw^*)\leq 4Lb^2\|\bw^*\|_2^2/n$. In this case, we have $\eta=1/(2L)$, $R\asymp n$ and choose $\gamma\asymp L$. Theorem \ref{thm:risk-mini} implies
  \begin{align*}
  \ebb[F(\bar{\bw}_{R})]-F(\bw^*)\lesssim \frac{F(\bw^*)}{b}+\frac{L\|\bw^*\|_2^2}{n}
    +
    L\Big(F(\bw^*)+\frac{L\|\bw^*\|_2^2}{n}\Big)\big(L^{-1}+LL^{-2}\big)\lesssim F(\bw^*)+\frac{L\|\bw^*\|_2^2}{n}.
  \end{align*}
  The proof is completed.
\end{proof}

\subsection{Proof of Theorem \ref{thm:stab-mini-sg} and Theorem \ref{thm:risk-mini-sg}\label{sec:proof-mini-sg}}
In this section, we prove stability and risk bounds for minibatch SGD applied to strongly convex problems.

\begin{proof}[Proof of Theorem \ref{thm:stab-mini-sg}]
For simplicity, we assume $f(\bw;z)=g(\bw;z)+r(\bw)$ with $r:\wcal\mapsto\rbb^+$ being $\mu$-strongly convex and $g:\wcal\times\zcal\mapsto\rbb^+$ being convex (this is a typical form for strongly convex problems in machine learning).
According to Eq. \eqref{sgd-e} and the sub-additivity of $\|\cdot\|_2$, we know
\begin{multline*}
  \|\bw_{t+1}-\bw_{t+1}^{(m)}\|_2 \leq \frac{\eta_t\alpha_{t,m}}{b}\|\nabla g(\bw_t;z_m)-\nabla g(\bw^{(m)}_t;z_m')\|_2+\\
  \big\|\bw_{t}-\frac{\eta_t}{b}\sum_{k:k\neq m}\alpha_{t,k}\nabla g(\bw_t;z_k)-\eta_t \nabla r(\bw_t)
   - \bw^{(m)}_{t}+\frac{\eta_t}{b}\sum_{k:k\neq m}\alpha_{t,k}\nabla g(\bw^{(m)}_t;z_k)+\eta_t\nabla r(\bw_t^{(m)})\big\|_2.
\end{multline*}
Since $f$ is $L$-smooth and $\sum_{k:k\neq m}\alpha_{t,k}\leq b$, we know the function $\bw\mapsto \frac{1}{b}\sum_{k:k\neq m}\alpha_{t,k} f(\bw;z_k)+r(\bw)$ is $L$-smooth and $\mu$-strongly convex.
By Lemma \ref{lem:nonexpansive} and the assumption $\eta_t\leq1/L$, we know
  \begin{equation}
  \|\bw_{t+1}-\bw_{t+1}^{(m)}\|_2 \leq (1-\mu\eta_t/2)\|\bw_{t}-\bw_{t}^{(m)}\|_2+\frac{\eta_t\alpha_{t,m}}{b}\|\nabla g(\bw_t;z_m)-\nabla g(\bw^{(m)}_t;z_m')\|_2\label{mini-12}.
  \end{equation}
  Taking an expectation over both sides yields (note $\bw_t,\bw_t^{(m)}$ are independent of $J_t$)
  \[
  \ebb[\|\bw_{t+1}-\bw_{t+1}^{(m)}\|_2]\leq (1-\mu\eta_t/2)\ebb[\|\bw_{t}-\bw_{t}^{(m)}\|_2]+\frac{2\eta_t\sqrt{2L\ebb[f(\bw_t;z_m)]}}{n},
  \]
  where we have used Eq. \eqref{mini-5} and Eq. \eqref{mini-6}.
  It then follows that
  \[
  \ebb[\|\bw_{t+1}-\bw_{t+1}^{(m)}\|_2]\leq\frac{2\sqrt{2L}}{n}\sum_{k=1}^{t}\eta_k\sqrt{\ebb[f(\bw_k;z_m)]}\prod_{k'=k+1}^{t}(1-\mu\eta_{k'}/2).
  \]
  We  take an average over $m$ to derive
  \begin{align*}
    \frac{1}{n}\sum_{m=1}^{n}\ebb[\|\bw_{t+1}-\bw_{t+1}^{(m)}\|_2] & \leq \frac{2\sqrt{2L}}{n^2}\sum_{k=1}^{t}\sum_{m=1}^{n}\eta_k\sqrt{\ebb[f(\bw_k;z_m)]}\prod_{k'=k+1}^{t}(1-\mu\eta_{k'}/2) \\
     & \leq \frac{2\sqrt{2L}}{n}\sum_{k=1}^{t}\eta_k\Big(\frac{1}{n}\sum_{m=1}^{n}\ebb[f(\bw_k;z_m)]\Big)^{\frac{1}{2}}\prod_{k'=k+1}^{t}(1-\mu\eta_{k'}/2)\\
     & = \frac{2\sqrt{2L}}{n}\sum_{k=1}^{t}\eta_k\sqrt{\ebb[F_S(\bw_k)]}\prod_{k'=k+1}^{t}(1-\mu\eta_{k'}/2),
  \end{align*}
  where we have used the concavity of $x\mapsto\sqrt{x}$. This proves Eq. \eqref{stab-mini-sg-a}.

  We now turn to Eq. \eqref{stab-mini-sg-b}.
  Let $\bw_S=\arg\min_{\bw\in\wcal}F_S(\bw)$. The following inequality was established in~\citep{woodworth2020minibatch}
  \[
  \ebb[\|\bw_{k+1}-\bw_S\|_2^2] \leq (1-\mu\eta_k)\ebb[\|\bw_k-\bw_S\|_2^2]-\eta_k\ebb[F_S(\bw_k)-F_S(\bw_S)]+\frac{2\eta_k^2\sigma_S^2}{b},
  \]
  where $\sigma_S^2=\ebb_{i_t}[\|\nabla f(\bw_S;z_{i_t})-\nabla F_S(\bw_S)\|_2^2]$. We  multiply both sides by $\prod_{k'=k+1}^{t}(1-\mu\eta_{k'}/2)$ and derive
  \begin{multline*}
  \prod_{k'=k+1}^{t}(1-\mu\eta_{k'}/2)\ebb[\|\bw_{k+1}-\bw_S\|_2^2] \leq \prod_{k'=k}^{t}(1-\mu\eta_{k'}/2)\ebb[\|\bw_k-\bw_S\|_2^2]-\\ \eta_k\prod_{k'=k+1}^{t}(1-\mu\eta_{k'}/2)\ebb[F_S(\bw_k)-F_S(\bw_S)]+\frac{2\sigma_S^2\eta_k^2\prod_{k'=k+1}^{t}(1-\mu\eta_{k'}/2)}{b}.
  \end{multline*}
  We take a summation of the above inequality and derive
  \begin{multline}\label{mini-8}
  \sum_{k=1}^{t}\eta_k\prod_{k'=k+1}^{t}(1-\mu\eta_{k'}/2)\ebb[F_S(\bw_k)-F_S(\bw_S)] \leq \ebb[\|\bw_1-\bw_S\|_2^2]\prod_{k'=1}^{t}(1-\mu\eta_{k'}/2) + \\ \frac{2\sigma_S^2}{b}\sum_{k=1}^{t}\eta_k^2\prod_{k'=k+1}^{t}(1-\mu\eta_{k'}/2).
  \end{multline}
  There holds
  \begin{align}
    \frac{\mu}{2}\sum_{k=1}^{t}\eta_k\prod_{k'=k+1}^{t}(1-\mu\eta_{k'}/2) & = \sum_{k=1}^{t}\big(1-(1-\mu\eta_k/2)\big)\prod_{k'=k+1}^{t}(1-\mu\eta_{k'}/2) \notag\\
    & = \sum_{k=1}^{t}\Big(\prod_{k'=k+1}^{t}(1-\mu\eta_{k'}/2)-\prod_{k'=k}^{t}(1-\mu\eta_{k'}/2)\Big) \notag\\
    & = 1 - \prod_{k'=1}^{t}(1-\mu\eta_{k'}/2)\leq 1.\label{mini-7}
  \end{align}
  By the strong convexity of $F_S$ and $\nabla F_S(\bw_S)=0$, we know
  \[
  F_S(\bw_1)-F_S(\bw_S)=F_S(\bw_1)-F_S(\bw_S)-\langle\bw_1-\bw_S,\nabla F_S(\bw_S)\rangle\geq\frac{\mu}{2}\|\bw_1-\bw_S\|_2^2
  \]
  and therefore
  \[
  \ebb[\|\bw_1-\bw_S\|_2^2]\leq \frac{2}{\mu}\ebb[F_S(\bw_1)-F_S(\bw_S)]\lesssim 1/\mu.
  \]
  We can plug the above inequality and Eq. \eqref{mini-7} back into Eq. \eqref{mini-8} to derive (note $\eta_k\leq 1/L$ and $\eta_k\mu\leq \mu/L\leq1$)
  \begin{multline*}
  \sum_{k=1}^{t}\eta_k\prod_{k'=k+1}^{t}(1-\mu\eta_{k'}/2)\ebb[F_S(\bw_k)-F_S(\bw_S)]
  \leq \ebb[\|\bw_1-\bw_S\|_2^2]+\frac{2\sigma_S^2}{bL}\sum_{k=1}^{t}\eta_k\prod_{k'=k+1}^{t}(1-\mu\eta_{k'}/2)\lesssim 1/\mu.
  \end{multline*}
  We combine the above inequality and Eq. \eqref{mini-7} together and derive
  \begin{align}
  \sum_{k=1}^{t}\eta_k\prod_{k'=k+1}^{t}(1-\mu\eta_{k'}/2)\ebb[F_S(\bw_k)]&=\ebb[F_S(\bw_S)]\sum_{k=1}^{t}\eta_k\prod_{k'=k+1}^{t}(1-\mu\eta_{k'}/2)\notag\\
  &+\sum_{k=1}^{t}\eta_k\prod_{k'=k+1}^{t}(1-\mu\eta_{k'}/2)\ebb[F_S(\bw_k)-F_S(\bw_S)]\lesssim 1/\mu.\label{stab-mini-4}
  \end{align}
  This together with Eq. \eqref{mini-7} implies that
  \[
    \sum_{k=1}^{t}\eta_k\sqrt{\ebb[F_S(\bw_k)]}\prod_{k'=k+1}^{t}(1-\mu\eta_{k'}/2)  \leq \frac{1}{2}\sum_{k=1}^{t}\eta_k\prod_{k'=k+1}^{t}(1-\mu\eta_{k'}/2)\big(1+\ebb[F_S(\bw_k)]\big) \\
      \lesssim 1/\mu.
  \]
  We  plug the above inequality back into Eq. \eqref{stab-mini-sg-a} to derive Eq. \eqref{stab-mini-sg-b}.

  Finally, we prove Eq. \eqref{stab-mini-sg-c}. Recall the notations in Eq. \eqref{mini-11}. Then, Eq. \eqref{mini-12} implies $\Delta_{t+1,m}\leq (1-\mu\eta_t/2)\Delta_{t,m}+\eta_t\alpha_{t,m}\mathfrak{C}_{t,m}/b$. We apply this inequality recursively, and get
  \begin{align*}
    \Delta_{t+1,m} & \leq \frac{1}{b}\sum_{k=1}^{t}\eta_k\alpha_{k,m}\mathfrak{C}_{k,m}\prod_{k'=k+1}^{t}(1-\mu\eta_{k'}/2) \\
     & = \frac{1}{b}\sum_{k=1}^{t}\eta_k\big(\alpha_{k,m}-b/n\big)\mathfrak{C}_{k,m}\prod_{k'=k+1}^{t}(1-\mu\eta_{k'}/2)+\frac{1}{n}\sum_{k=1}^{t}\eta_k\mathfrak{C}_{k,m}\prod_{k'=k+1}^{t}(1-\mu\eta_{k'}/2).
  \end{align*}
  We take a square and an expectation over both sides, and get
  \begin{align*}
  & \ebb[\Delta_{t+1,m}^2] \\
  &\leq \frac{2}{b^2}\ebb\Big[\Big(\sum_{k=1}^{t}\eta_k\big(\alpha_{k,m}-b/n\big)\mathfrak{C}_{k,m}\prod_{k'=k+1}^{t}(1-\mu\eta_{k'}/2)\Big)^2\Big]
  +\frac{2}{n^2}\ebb\Big[\Big(\sum_{k=1}^{t}\eta_k\mathfrak{C}_{k,m}\prod_{k'=k+1}^{t}(1-\mu\eta_{k'}/2)\Big)^2\Big]\\
  & =\frac{2}{b^2}\sum_{k=1}^{t}\eta_k^2\ebb\Big[\big(\alpha_{k,m}-b/n\big)^2\mathfrak{C}^2_{k,m}\prod_{k'=k+1}^{t}(1-\mu\eta_{k'}/2)^2\Big]
  +\frac{2}{n^2}\ebb\Big[\Big(\sum_{k=1}^{t}\eta_k\mathfrak{C}_{k,m}\prod_{k'=k+1}^{t}(1-\mu\eta_{k'}/2)\Big)^2\Big]\\
  & \leq \frac{2}{nb}\sum_{k=1}^{t}\eta_k^2\ebb\big[\mathfrak{C}^2_{k,m}\big]\prod_{k'=k+1}^{t}(1-\mu\eta_{k'}/2)^2
  +\frac{2}{n^2}\ebb\Big[\Big(\sum_{k=1}^{t}\eta_k\mathfrak{C}_{k,m}\prod_{k'=k+1}^{t}(1-\mu\eta_{k'}/2)\Big)^2\Big],
  \end{align*}
  where we have used Eq. \eqref{mini-14} and $\ebb_{J_k}[\big(\alpha_{k,m}-b/n\big)^2]=\var(\alpha_{k,m})\leq b/n$.
  Furthermore, by the Schwarz's inequality and Eq. \eqref{mini-7}, we know
  \begin{align*}
    \Big(\sum_{k=1}^{t}\eta_k\mathfrak{C}_{k,m}\prod_{k'=k+1}^{t}(1-\mu\eta_{k'}/2)\Big)^2 & \leq \Big(\sum_{k=1}^{t}\eta_k\mathfrak{C}_{k,m}^2\prod_{k'=k+1}^{t}(1-\mu\eta_{k'}/2)\Big)\Big(\sum_{k=1}^{t}\eta_k\prod_{k'=k+1}^{t}(1-\mu\eta_{k'}/2)\Big) \\
     & \leq \frac{2}{\mu}\sum_{k=1}^{t}\eta_k\mathfrak{C}_{k,m}^2\prod_{k'=k+1}^{t}(1-\mu\eta_{k'}/2).
  \end{align*}
  We can combine the above two inequalities together and derive
  \[
  \ebb[\Delta_{t+1,m}^2]\leq \sum_{k=1}^{t}\Big(\frac{2\eta_k^2}{nb}+\frac{4\eta_k}{n^2\mu}\Big)\ebb\big[\mathfrak{C}^2_{k,m}\big]\prod_{k'=k+1}^{t}(1-\mu\eta_{k'}/2).
  \]
  Analogous to Eq. \eqref{mini-5}, we know $\ebb[\mathfrak{C}_{k,m}^2]\leq 8L\ebb[f(\bw_k;z_m)]$ and therefore
  \[
  \ebb[\Delta_{t+1,m}^2]\leq \sum_{k=1}^{t}\Big(\frac{16L\eta_k^2}{nb}+\frac{32L\eta_k}{n^2\mu}\Big)\ebb[f(\bw_k;z_m)]\prod_{k'=k+1}^{t}(1-\mu\eta_{k'}/2).
  \]
  We can take an average over $m\in[n]$ to get the stated bound. The proof is completed.

\end{proof}

\begin{proof}[Proof of Theorem \ref{thm:risk-mini-sg}]
  Since $F_S(\bw_S)\leq F_S(\bw^*)$, an upper bound on $F_S(A(S))-F_S(\bw_S)$ is also an upper bound on $F_S(A(S))-F_S(\bw^*)$. Then, according to \citep{woodworth2020minibatch}, there exists an average $\hat{\bw}_R$ of $\{\bw_t\}$ such that
  \begin{equation}\label{mini-9} 
    \ebb_A[F_S(\hat{\bw}_R)]-F_S(\bw^*)\lesssim \frac{L}{\mu}\exp\big(-\mu R/L\big)+\frac{\sigma_*^2}{\mu bR}.
  \end{equation}
  Theorem \ref{thm:stab-mini-sg} shows that an algorithm outputting any iterate produced by Eq. \eqref{sgd} would be $\ell_1$ on-average model $O(1/(n\mu))$-stable. It then follows that the output model $\hat{\bw}_R$ would also be $\ell_1$ on-average model $O(1/(n\mu))$-stable. Lemma \ref{lem:gen-model-stab} (Part (a)) then implies
  \[
    \ebb[F(\hat{\bw}_{R})-F_S(\hat{\bw}_{R})]\lesssim G/(n\mu).
  \]
  We  plug the above two inequalities back to Eq. \eqref{decomposition} and derive
    \[
    \ebb[F(\hat{\bw}_{R})]-F(\bw^*)\lesssim \frac{L}{\mu}\exp\big(-\mu R/L\big)+\frac{\sigma_*^2}{\mu bR}+\frac{G}{n\mu}.
  \]
  If we choose $R> \frac{L}{\mu}\log\frac{nL}{G}$ and $b>\frac{n\sigma_*^2}{GR}$, we get
  \[
  \frac{L}{\mu}\exp\big(-\mu R/L\big)\lesssim G/n\mu\quad\text{and}\quad \frac{\sigma_*^2}{\mu bR}\lesssim G/n\mu.
  \]
  The proof is completed.
\end{proof}

\subsection{Proof of Theorem \ref{thm:stab-mini-nonconvex} and  Theorem \ref{thm:risk-mini-pl}\label{sec:stab-nonconvex}}
In this section, we present the proof of minibatch SGD for nonconvex problems. We first prove Theorem~\ref{thm:stab-mini-nonconvex}.
\begin{proof}[Proof of Theorem \ref{thm:stab-mini-nonconvex}]
According to Eq. \eqref{stab-mini-1a} and the smoothness of $f$, we know
$
\Delta_{t+1,m} \leq (1+\eta_tL)\Delta_{t,m} +\frac{\eta_t\alpha_{t,m}\mathfrak{C}_{t,m}}{b}.
$
We apply the above inequality recursively, and derive
\[
\ebb[\Delta_{t+1,m}]\leq \sum_{k=1}^{t}\frac{\eta_k\ebb[\alpha_{k,m}\mathfrak{C}_{k,m}]}{b}\prod_{k'=k+1}^{t}(1+\eta_{k'}L)
= \sum_{k=1}^{t}\frac{\eta_k\ebb[\mathfrak{C}_{k,m}]}{n}\prod_{k'=k+1}^{t}(1+\eta_{k'}L).
\]
Analogous to Eq. \eqref{sum-concave}, we then get
\begin{align*}
  \frac{1}{n}\sum_{m=1}^{n}\ebb[\|\bw_{t+1}-\bw_{t+1}^{(m)}\|_2] & \leq \frac{2\sqrt{2L}}{n}\sum_{m=1}^{n}\sum_{k=1}^{t}\frac{\eta_k\ebb\big[\sqrt{f(\bw_k;z_m)}\big]}{n}\prod_{k'=k+1}^{t}(1+\eta_{k'}L) \\
   & \leq \frac{2\sqrt{2L}}{n}\sum_{k=1}^{t}\eta_k\ebb\big[\sqrt{F_S(\bw_k)}\big]\prod_{k'=k+1}^{t}(1+\eta_{k'}L).
\end{align*}
The proof is completed.
\end{proof}

We now prove Theorem \ref{thm:risk-mini-pl} on risk bounds of minibatch SGD under the PL condition. We first introduce the following lemma relating generalization to optimization for problems under the PL condition~\citep{lei2021sharper}.

\begin{lemma}[Generalization Bounds under PL Condition\label{lem:gen-pl}]
Assume for all $z\in\zcal$, the map $\bw\mapsto f(\bw;z)$ is nonnegative and $L$-smooth. Let $A$ be an algorithm. If Assumption \ref{ass:pl} holds and $L\leq n\mu/4$, then 
\begin{equation}\label{gen-pl-a}
\ebb\big[F(A(S)) - F_S(A(S))]\leq
\frac{16L\ebb[F_S(A(S))]}{n\mu} + \frac{L\ebb\big[F_S(A(S))-F_S(\bw_S)\big]}{2\mu}.
\end{equation}
\end{lemma}

The following lemma gives the optimization error bounds for minibatch SGD under the PL condition.
\begin{lemma}[Optimization Errors for Minibatch SGD: PL condition\label{lem:opt-sgd-pl}]
  Assume for all $z\in\zcal$, the map $\bw\mapsto f(\bw;z)$ is nonnegative and $L$-smooth. Let Assumption \ref{ass:pl} hold and $\ebb_{i_k}\big[\|\nabla f(\bw_t;z_{i_k})-\nabla F_S(\bw_t)\|_2^2\big]\leq\sigma^2$, where $i_k$ follows from the uniform distribution over $[n]$.
  Let $\{\bw_t\}$ be produced by the algorithm $A$ defined in \eqref{sgd} with $\eta_t=2/(\mu(t+a))$ and $a\geq4L/\mu$. Then
  \begin{equation}\label{opt-sgd-pl}
    \ebb_A[F_S(\bw_{R+1})] - F_S(\bw_S) \lesssim \frac{L^2}{\mu^2 R^2} + \frac{L\sigma^2}{b\mu^2R}.
  \end{equation}
\end{lemma}
\begin{proof}
Note the assumption $a\geq4L/\mu$ implies $\eta_t\leq1/(2L)$.
For simplicity, we denote $g_t=\frac{1}{b}\sum_{j=1}^{b}\nabla f(\bw_t;z_{i_{t,j}})$. Then Eq. \eqref{sgd} becomes $\bw_{t+1}=\bw_t-\eta_tg_t$. By the $L$-smoothness of $F_S$, we have
\begin{align*}
  F_S(\bw_{t+1}) & \leq F_S(\bw_t)+\langle\bw_{t+1}-\bw_t,\nabla F_S(\bw_t)\rangle + \frac{L}{2}\|\bw_{t+1}-\bw_t\|_2^2 \\
   & = F_S(\bw_t) - \eta_t\langle g_t,\nabla F_S(\bw_t)\rangle + \frac{L\eta_t^2}{2}\|g_t\|_2^2\\
   & \leq F_S(\bw_t) - \eta_t\langle g_t,\nabla F_S(\bw_t)\rangle + L\eta_t^2\big(\|g_t-\nabla F_S(\bw_t)\|_2^2+\|\nabla F_S(\bw_t)\|_2^2\big).
\end{align*}
We take a conditional expectation over both sides and derive
\begin{align*}
\ebb_{J_t}[F_S(\bw_{t+1})] &\leq F_S(\bw_t) - \eta_t\|\nabla F_S(\bw_t)\|_2^2 + L\eta_t^2\big(\ebb_{J_t}[\|g_t-\nabla F_S(\bw_t)\|_2^2]+\|\nabla F_S(\bw_t)\|_2^2\big)\\
& = F_S(\bw_t) - \eta_t\|\nabla F_S(\bw_t)\|_2^2 / 2 + L\eta_t^2\ebb_{J_t}[\|g_t-\nabla F_S(\bw_t)\|_2^2],
\end{align*}
where we have used the assumption $\eta_t\leq 1/(2L)$. Note the variance reduces by a factor of $b$ with a minibatch, i.e.,
\[
\ebb_{J_t}[\|g_t-\nabla F_S(\bw_t)\|_2^2]=\frac{1}{b}\ebb_{i_k}[\|\nabla f(\bw_t;z_{i_k})-\nabla F_S(\bw_t)\|_2^2]\leq \frac{\sigma^2}{b}.
\]
We  combine the above two inequalities together and take an expectation w.r.t. the remaining random variables to get
\[
\ebb_A[F_S(\bw_{t+1})] \leq  \ebb_A[F_S(\bw_t)] - \eta_t\ebb_A[\|\nabla F_S(\bw_t)\|_2^2] / 2 + \frac{L\eta_t^2\sigma^2}{b}.
\]
We  subtract both sides by $F_S(\bw_S)$ and use the PL condition to derive
\begin{align*}
  \ebb_A[F_S(\bw_{t+1})] - F_S(\bw_S) &\leq  \ebb_A[F_S(\bw_t)] - F_S(\bw_S) - \mu\eta_t\big(\ebb_A[F_S(\bw_t)] - F_S(\bw_S)\big) + \frac{L\eta_t^2\sigma^2}{b}\\
  &  = (1 - \mu\eta_t)\big(\ebb_A[F_S(\bw_t)] - F_S(\bw_S)\big) + \frac{L\eta_t^2\sigma^2}{b}.
\end{align*}
Since $\eta_t=\frac{2}{\mu(a+t)}$, we know
\[
  \ebb_A[F_S(\bw_{t+1})] - F_S(\bw_S)  \leq \Big(1 - \frac{2}{a+t}\Big)\big(\ebb_A[F_S(\bw_t)] - F_S(\bw_S)\big) + \frac{4L\sigma^2}{b\mu^2(a+t)^2}.
\]
We multiply both sides by $(t+a)(t+a-1)$ and get
\[
(t+a)(t+a-1)\big(\ebb_A[F_S(\bw_{t+1})] - F_S(\bw_S)\big)  \leq (t+a-1)(t+a-2)\big(\ebb_A[F_S(\bw_t)] - F_S(\bw_S)\big) + \frac{4L\sigma^2}{b\mu^2}.
\]
We  take a summation of the above inequality from $t=1$ to $R$ and get
\[
(R+a)(R+a-1)\big(\ebb_A[F_S(\bw_{R+1})] - F_S(\bw_S)\big)  \leq a(a-1)\big(\ebb_A[F_S(\bw_1)] - F_S(\bw_S)\big) + \frac{4LR\sigma^2}{b\mu^2}.
\]
The stated bound then follows directly since $a\geq4L/\mu$. The proof is completed.
\end{proof}

Now we are ready to prove Theorem \ref{thm:risk-mini-pl} for nonconvex problems.
\begin{proof}[Proof of Theorem \ref{thm:risk-mini-pl}]
  According to Lemma \ref{lem:gen-pl} and Lemma \ref{lem:opt-sgd-pl}, we know
  \begin{align*}
  \ebb[F(\bw_R)-F_S(\bw_S)]&\lesssim \frac{L}{n\mu}+\frac{L\ebb[F_S(\bw_R)-F_S(\bw_S)]}{\mu}\lesssim \frac{L}{n\mu}+\frac{L^3}{\mu^3R^2}+\frac{L^2\sigma^2}{b\mu^3R}.
  \end{align*}
  Since $F_S(\bw_S)\leq F_S(\bw^*)$, we then derive
  \[
  \ebb[F(\bw_R)]-F(\bw^*)=\ebb[F(\bw_R)-F_S(\bw^*)]\leq\ebb[F(\bw_R)-F_S(\bw_S)]\lesssim \frac{L}{n\mu}+\frac{L^3}{\mu^3R^2}+\frac{L^2\sigma^2}{b\mu^3R}.
  \]
  Since $R\geq \max\big\{L\sqrt{n}/\mu,nL\sigma^2/(b\mu^2)\big\}$, we know
  \[
  \mu^2R^2\geq n\quad\text{and}\quad b\mu^2R\geq n.
  \]
  It then follows that $\ebb[F(\bw_R)]-F(\bw^*)\lesssim L/(n\mu)$. The proof is completed.
\end{proof}

\section{Proofs on Local SGD\label{sec:proof-local-sgd}}

\subsection{Proof of Theorem \ref{thm:stab-local}\label{sec:proof-stab-local}}
In this section, we prove stability bounds on local SGD.
\begin{proof}[Proof of Theorem \ref{thm:stab-local}]
Let $\{\bw_{m,r,t+1}^{(k)}\}$ be the sequence produced by Eq. \eqref{local-sgd} on $S^{(k)}$.
We introduce the notations
\[
\Delta_{m,r,t,k}=\big\|\bw_{m,r,t}-\bw_{m,r,t}^{(k)}\big\|_2,\quad \mathfrak{C}_{m,r,t,k}=\|\nabla f(\bw_{m,r,t};z_k)-\nabla f(\bw_{m,r,t}^{(k)};z_k')\|_2.
\]
If $i_{m,r,t}\neq k$, we can use Lemma \ref{lem:nonexpansive} to derive
\[
\Delta_{m,r,t+1,k}  = \big\|\bw_{m,r,t}-\eta_{r,t}\nabla f(\bw_{m,r,t};z_{i_{m,r,t}})-\bw^{(k)}_{m,r,t}+\eta_{r,t}\nabla f(\bw^{(k)}_{m,r,t};z_{i_{m,r,t}})\big\|_2\notag\\
\leq \|\bw_{m,r,t}-\bw_{m,r,t}^{(k)}\|_2.
\]
If $i_{m,r,t}=k$, we have
\[
\Delta_{m,r,t+1,k}  = \big\|\bw_{m,r,t}-\eta_{r,t}\nabla f(\bw_{m,r,t};z_k)-\bw_{m,r,t}^{(k)}+\eta_{r,t}\nabla f(\bw_{m,r,t}^{(k)};z_k')\big\|_2
\leq \Delta_{m,r,t,k}+\eta_{r,t}\mathfrak{C}_{m,r,t,k}.
\]
We combine the above two cases together and derive
\begin{equation}\label{local-7}
\Delta_{m,r,t+1,k}\leq \Delta_{m,r,t,k}+\eta_{r,t}\mathfrak{C}_{m,r,t,k}\ibb_{[i_{m,r,t}=k]},
\end{equation}
where $\ibb_{[i_{m,r,t}=k]}$ denotes the indicator function of the event $\{i_{m,r,t}=k\}$, i.e., $\ibb_{[i_{m,r,t}=k]}=1$ if $i_{m,r,t}=k$, and $0$ otherwise.
We apply the above inequality recursively and get
\[
\Delta_{m,r,K+1,k} \leq \Delta_{m,r,1,k}+\sum_{t=1}^{K}\eta_{r,t}\mathfrak{C}_{m,r,t,k}\ibb_{[i_{m,r,t}=k]}.
\]
We take an average over $m\in[M]$ and use $\bw_{r+1}=\frac{1}{M}\sum_{m=1}^{M}\bw_{m,r,K+1}$ to derive
\begin{equation}\label{local-8}
\big\|\bw_{r+1}-\bw_{r+1}^{(k)}\big\|_2  \leq \frac{1}{M}\sum_{m=1}^{M}\|\bw_{m,r,K+1}-\bw_{m,r,K+1}^{(k)}\|_2
\leq \big\|\bw_r-\bw_r^{(k)}\big\|_2+\sum_{m=1}^{M}\sum_{t=1}^{K}\frac{\eta_{r,t}}{M}\mathfrak{C}_{m,r,t,k}\ibb_{[i_{m,r,t}=k]},
\end{equation}
where we have used $\bw_{m,r,1}=\bw_r$.
We can apply the above inequality recursively, and derive
\begin{equation}\label{local-4}
\big\|\bw_{R+1}-\bw_{R+1}^{(k)}\big\|_2\leq \sum_{r=1}^{R}\sum_{m=1}^{M}\sum_{t=1}^{K}\frac{\eta_{r,t}}{M}\mathfrak{C}_{m,r,t,k}\ibb_{[i_{m,r,t}=k]}.
\end{equation}

We first consider the $\ell_1$ on-average model stability.
We know that $i_{m,r,t}$ takes the value $k$ with probability $1/n$, and other values with probability $1-1/n$.
We  take expectation w.r.t. $i_{m,r,t}$ and note $\mathfrak{C}_{m,r,t,k}$ is independent of $i_{m,r,t}$, which implies
\begin{equation}\label{local-3}
\ebb\big[\|\bw_{R+1}-\bw^{(k)}_{R+1}\|_2\big] \leq \sum_{r=1}^{R}\sum_{m=1}^{M}\sum_{t=1}^{K}\frac{\eta_{r,t}}{nM}\ebb[\mathfrak{C}_{m,r,t,k}] \leq \frac{2\sqrt{2L}}{nM}\sum_{r=1}^{R}\sum_{m=1}^{M}\sum_{t=1}^{K}\eta_{r,t}\ebb\Big[\sqrt{f(\bw_{m,r,t};z_k)}\Big],
\end{equation}
where we have used the self-bounding property and the symmetry between $z_k$ and $z_k'$ (analogous to Eq. \eqref{mini-5}).
It then follows from the concavity of $x\mapsto\sqrt{x}$ that
\begin{align}
  \frac{1}{n}\sum_{k=1}^{n}\ebb\big[\|\bw_{R+1}-\bw^{(k)}_{R+1}\|_2\big] & \leq \frac{2\sqrt{2L}}{n^2M}\sum_{k=1}^{n}\sum_{r=1}^{R}\sum_{m=1}^{M}\sum_{t=1}^{K}\eta_{r,t}\ebb\Big[\sqrt{f(\bw_{m,r,t};z_k)}\Big]\notag\\
  & \leq \frac{2\sqrt{2L}}{nM}\sum_{r=1}^{R}\sum_{m=1}^{M}\sum_{t=1}^{K}\eta_{r,t}\ebb\Big[\Big(\frac{1}{n}\sum_{k=1}^{n}f(\bw_{m,r,t};z_k)\Big)^{\frac{1}{2}}\Big]\notag\\
  & = \frac{2\sqrt{2L}}{nM}\sum_{r=1}^{R}\sum_{m=1}^{M}\sum_{t=1}^{K}\eta_{r,t}\ebb\Big[\Big(F_S(\bw_{m,r,t})\Big)^{\frac{1}{2}}\Big].\notag
\end{align}
This proves Eq. \eqref{local-l1}. We now consider the $\ell_2$ on-average model stability.
We take an expectation-variance decomposition in Eq. \eqref{local-4} and derive
\begin{equation}\label{local-5}
\big\|\bw_{R+1}-\bw_{R+1}^{(k)}\big\|_2\leq \sum_{r=1}^{R}\sum_{m=1}^{M}\sum_{t=1}^{K}\frac{\eta_{r,t}}{M}\mathfrak{C}_{m,r,t,k}\big(\ibb_{[i_{m,r,t}=k]}-1/n\big)+n^{-1}\sum_{r=1}^{R}\sum_{m=1}^{M}\sum_{t=1}^{K}\frac{\eta_{r,t}}{M}\mathfrak{C}_{m,r,t,k}.
\end{equation}
Analogous to Eq. \eqref{mini-14}, we have (note $i_{m,r,t}$ is independent of $i_{m',r',t'}$ if $(m,r,t)\neq (m',r',t')$,
$\mathfrak{C}_{m,r,t,k}$ is independent of $i_{m,r,t}$, and $\mathfrak{C}_{m',r',t',k}$ is independent of $i_{m',r',t'}$)
\[
\ebb\Big[\mathfrak{C}_{m,r,t,k}\big(\ibb_{[i_{m,r,t}=k]}-1/n\big)\mathfrak{C}_{m',r',t',k}\big(\ibb_{[i_{m',r',t'}=k]}-1/n\big)\Big]=0\quad\text{if either }t\neq t', m\neq m', \text{ or }r\neq r'.
\]
Then, we take a square on both sides of Eq. \eqref{local-5} followed by expectation, and analyze analogously to the proof of Eq. \eqref{on-average-b}:
\begin{align*}
  & \ebb\big[\big\|\bw_{R+1}-\bw_{R+1}^{(k)}\big\|_2^2\big] \\
  & \leq 2\ebb\Big[\Big(\sum_{r=1}^{R}\sum_{m=1}^{M}\sum_{t=1}^{K}\frac{\eta_{r,t}}{M}\mathfrak{C}_{m,r,t,k}\big(\ibb_{[i_{m,r,t}=k]}-1/n\big)\Big)^2\Big] +
  \frac{2}{n^2}\ebb\Big[\Big(\sum_{r=1}^{R}\sum_{m=1}^{M}\sum_{t=1}^{K}\frac{\eta_{r,t}}{M}\mathfrak{C}_{m,r,t,k}\Big)^2\Big]\\
   & = \frac{2}{M^2}\ebb\Big[\sum_{r=1}^{R}\sum_{m=1}^{M}\sum_{t=1}^{K}\eta^2_{r,t}\mathfrak{C}^2_{m,r,t,k}\var(\ibb_{[i_{m,r,t}=k]}\big)\Big]  +
  \frac{2}{n^2M^2}\ebb\Big[\Big(\sum_{r=1}^{R}\sum_{m=1}^{M}\sum_{t=1}^{K}\eta_{r,t}\mathfrak{C}_{m,r,t,k}\Big)^2\Big]\\
  & \leq \frac{2}{nM^2}\ebb\Big[\sum_{r=1}^{R}\sum_{m=1}^{M}\sum_{t=1}^{K}\eta^2_{r,t}\mathfrak{C}^2_{m,r,t,k})\Big]  +
  \frac{2}{n^2M^2}\ebb\Big[\Big(\sum_{r=1}^{R}\sum_{m=1}^{M}\sum_{t=1}^{K}\eta_{r,t}\mathfrak{C}_{m,r,t,k}\Big)^2\Big],
\end{align*}
where we have used $\var(\ibb_{[i_{m,r,t}=k]})\leq1/n$. By the self-bounding property of $f$ we know
\begin{equation}\label{local-6}
\ebb\big[\mathfrak{C}_{m,r,t,k}^2\big]\leq 4L\ebb\big[f(\bw_{m,r,t};z_k)+f(\bw^{(k)}_{m,r,t};z_k')\big]=8L\ebb[f(\bw_{m,r,t};z_k)].
\end{equation}
It then follows that
\begin{equation}\label{local-9}
\ebb\big[\big\|\bw_{R+1}-\bw_{R+1}^{(k)}\big\|_2^2\big]\leq \frac{16L}{nM^2}\ebb\Big[\sum_{r=1}^{R}\sum_{m=1}^{M}\sum_{t=1}^{K}\eta^2_{r,t}f(\bw_{m,r,t};z_k)\Big]  +
  \frac{2}{n^2M^2}\ebb\Big[\Big(\sum_{r=1}^{R}\sum_{m=1}^{M}\sum_{t=1}^{K}\eta_{r,t}\mathfrak{C}_{m,r,t,k}\Big)^2\Big].
\end{equation}
The stated bound then follows by taking an average over $k\in[n]$ and noting $F_S(\bw)=\frac{1}{n}\sum_{k=1}^{n}f(\bw;z_k)$.
The proof is completed.
\end{proof}

\subsection{Proof of Theorem \ref{thm:gen-local}\label{sec:gen-local}}
In this section, we prove Theorem \ref{thm:gen-local} on excess population risk bounds of local SGD for convex problems. To this aim, we require the following lemma on the optimization error bounds~\citep{woodworth2020local}. Note $F_S(\bw^*)\geq F_S(\bw_S)$.
\begin{lemma}[Optimization Errors of Local SGD: Convex Case\label{lem:opt-local}] 
Assume for all $z\in\zcal$, the map $\bw\mapsto f(\bw;z)$ is nonnegative, convex and $L$-smooth. Let $\{\bw_{m,r,t}\}$ be produced by the algorithm $A$ defined in \eqref{local-sgd} with $\eta\leq1/(4L)$.  Assume for all $r\in[R],t\in[K]$, $\ebb_{i_{m,r,t}}[\|\nabla f(\bw_{r,t};z_{i_{m,r,t}})-\nabla F_S(\bw_{r,t})\|_2^2]\leq\sigma^2$.
  Then the following inequality holds
  \begin{equation}\label{opt-local}
  \ebb_A[F_S(\bar{\bw}_{R,1})]-F_S(\bw^*)\lesssim\frac{\|\bw^*\|_2^2}{\eta KR}+\frac{\eta\sigma^2}{M}+L(K-1)\eta^2\sigma^2.
  \end{equation}
\end{lemma}

We are now ready to prove Theorem \ref{thm:gen-local}. For simplicity, we assume $\ebb\big[\sqrt{F_S(\bw_{m',r',t'})}\big]\lesssim 1$, which is reasonable since we are minimizing $F_S$ by local SGD. Note this assumption is used to bound the stability and can be removed if we assume $f$ is Lipschitz continuous (${F_S(\bw_{m',r',t'})}$ appears in the stability analysis since we control the gradient norm by function values).
\begin{proof}[Proof of Theorem \ref{thm:gen-local}]
  Analogous to Eq. \eqref{local-9}, one can show that
  \begin{multline*}
  \ebb\Big[\Big\|\frac{1}{M}\sum_{m=1}^{M}\bw_{m,r,t}-\sum_{m=1}^{M}\bw_{m,r,t}^{(k)}\Big\|_2^2\Big]\leq \frac{16L}{nM^2}\ebb\Big[\sum_{r'=1}^{r}\sum_{m=1}^{M}\sum_{t'=1}^{K}\eta^2_{r',t'}f(\bw_{m,r',t'};z_k)\Big] \\ +
  \frac{2}{n^2M^2}\ebb\Big[\Big(\sum_{r'=1}^{r}\sum_{m=1}^{M}\sum_{t'=1}^{K}\eta_{r',t'}\mathfrak{C}_{m,r',t',k}\Big)^2\Big].
  \end{multline*}
  We take an average over $k\in[n]$, and derive
  \begin{multline*}
    \frac{1}{n}\sum_{k=1}^{n}\ebb\Big[\Big\|\frac{1}{M}\sum_{m=1}^{M}\bw_{m,r,t}-\sum_{m=1}^{M}\bw_{m,r,t}^{(k)}\Big\|_2^2\Big]\leq \frac{16L}{n^2M^2}\sum_{k=1}^{n}\sum_{r'=1}^{r}\sum_{m=1}^{M}\sum_{t'=1}^{K}\eta^2_{r',t'}\ebb\big[f(\bw_{m,r',t'};z_k)\big] \\ +
  \frac{2}{n^3M^2}\sum_{k=1}^{n}rMK\sum_{r'=1}^{r}\sum_{m=1}^{M}\sum_{t'=1}^{K}\eta^2_{r',t'}\ebb[\mathfrak{C}^2_{m,r',t',k}].
  \end{multline*}
  By the self-bounding property and the symmetry between $z_k$ and $z_k'$, we further know
  \[
  \ebb[\mathfrak{C}^2_{m,r',t',k}]\leq 2\ebb\big[\|\nabla f(\bw_{m,r',t'};z_k)\|_2^2]+\ebb\big[\|\nabla f(\bw_{m,r',t'}^{(k)};z_k')\|_2^2\big]\leq 8L\ebb[f(\bw_{m,r',t'};z_k)].
  \]
  It then follows that
  \begin{multline*}
    \frac{1}{n}\sum_{k=1}^{n}\ebb\Big[\Big\|\frac{1}{M}\sum_{m=1}^{M}\bw_{m,r,t}-\sum_{m=1}^{M}\bw_{m,r,t}^{(k)}\Big\|_2^2\Big]\leq \frac{16L}{nM^2}\sum_{r'=1}^{r}\sum_{m=1}^{M}\sum_{t'=1}^{K}\eta^2_{r',t'}\ebb\big[F_S(\bw_{m,r',t'})\big] \\ +
  \frac{16LrK}{n^2M}\sum_{r'=1}^{r}\sum_{m=1}^{M}\sum_{t'=1}^{K}\eta^2_{r',t'}\ebb\big[F_S(\bw_{m,r',t'})\big].
  \end{multline*}
  It then follows the convexity of $\|\cdot\|^2$ that
  \begin{align*}
    & \frac{1}{n}\sum_{k=1}^{n}\ebb\big[\|\bar{\bw}_{R,1}-\bar{\bw}_{R,1}^{(k)}\|_2^2\big] \leq \frac{1}{KRn}\sum_{k=1}^{n}\sum_{r=1}^{R}\sum_{t=1}^{K}\ebb\Big[\Big\|\frac{1}{M}\sum_{m=1}^{M}\bw_{m,r,t}-\sum_{m=1}^{M}\bw_{m,r,t}^{(k)}\Big\|_2^2\Big]\\
     & \lesssim  \frac{1}{KR}\sum_{r=1}^{R}\sum_{t=1}^{K}\frac{L\eta^2}{nM}\Big(\frac{1}{M}+\frac{rK}{n}\Big)\sum_{r'=1}^{r}\sum_{m'=1}^{M}\sum_{t'=1}^{K}\ebb\big[F_S(\bw_{m',r',t'})\big]\\
     & \lesssim\frac{L\eta^2}{KRnM}\Big(\frac{KR}{M}+\frac{K^2R^2}{n}\Big)\sum_{r=1}^{R}\sum_{m=1}^{M}\sum_{t=1}^{K}\ebb\big[F_S(\bw_{m,r,t})\big].
  \end{align*}
  According to Lemma \ref{lem:gen-model-stab} (Part (b)) and using the assumption $\ebb\big[F_S(\bw_{m,r,t})\big]\lesssim 1$, we know
  \[
  \ebb[F(\bar{\bw}_{R,1})-F_S(\bar{\bw}_{R,1})] \lesssim \frac{L}{\gamma}\ebb[F_S(\bar{\bw}_{R,1})]+
  \frac{L(L+\gamma)\eta^2}{n}\Big(\frac{KR}{M}+\frac{K^2R^2}{n}\Big).
  \]
  We  combine the above inequality and Lemma \ref{lem:opt-local} together, and derive
  \begin{multline*}
  \ebb[F(\bar{\bw}_{R,1})]-F(\bw^*)\lesssim \frac{L}{\gamma}\ebb[F_S(\bar{\bw}_{R,1})]+
  \frac{L(L+\gamma)\eta^2}{n}\Big(\frac{KR}{M}+\frac{K^2R^2}{n}\Big)
  +\frac{\|\bw^*\|_2^2}{\eta KR}+\frac{\eta\sigma^2}{M}+L(K-1)\eta^2\sigma^2.
  \end{multline*}
  We can minimize $\gamma$ and use $KRM\asymp n$ to get
  \[
  \ebb[F(\bar{\bw}_{R,1})]-F(\bw^*)\lesssim \frac{LKR\eta}{n}+\frac{L^2\eta^2K^2R^2}{n^2}
  +\frac{\|\bw^*\|_2^2}{\eta KR}+\frac{\eta\sigma^2}{M}+L(K-1)\eta^2\sigma^2.
  \]
  Since $\eta\asymp \|\bw^*\|_2\sqrt{n}/(KR\sqrt{L})$, we know
  \begin{gather*}
  \frac{LKR\eta}{n}\asymp \frac{\|\bw^*\|_2^2}{\eta KR}\asymp \frac{\sqrt{L}\|\bw^*\|_2}{\sqrt{n}},\\
  \frac{L^2\eta^2K^2R^2}{n^2}\asymp \frac{L^2\|\bw^*\|_2^2nK^2R^2}{n^2K^2R^2L}=\frac{L\|\bw^*\|_2^2}{n},\\
  \frac{\eta\sigma^2}{M}\asymp \frac{\|\bw^*\|_2\sqrt{n}\sigma^2}{MKR\sqrt{L}}\asymp \frac{\|\bw^*\|_2\sigma^2}{\sqrt{nL}}.
  \end{gather*}
  Since $\eta\lesssim (K-1)^{-\frac{1}{2}}\|\bw^*\|_2^{\frac{1}{2}}/(nL)^{\frac{1}{4}}$, we further know
  \[
  L(K-1)\eta^2\sigma^2\asymp \frac{\sqrt{L}\|\bw^*\|_2}{\sqrt{n}}.
  \]
  The stated bound then follows by combining the above discussions together.
\end{proof}

\subsection{Proof of Theorem \ref{thm:risk-local-sg}\label{sec:proof-risk-local-sg}}
To prove Theorem \ref{thm:risk-local-sg}, we require the following lemma on optimization errors~\citep{stich2018local,khaled2020tighter}. 
\begin{lemma}[Optimization Errors of Local SGD: Strongly Convex Case\label{lem:opt-local-sg}] 
Assume for all $z\in\zcal$, the map $\bw\mapsto f(\bw;z)$ is nonnegative, $\mu$-strongly convex and $L$-smooth. Let $\{\bw_{m,r,t}\}$ be produced by the algorithm $A$ defined in \eqref{local-sgd} with $\eta_{r,t}=\frac{4}{\mu(a+(r-1)K+t)}\leq2/L$ with $a>2L/\mu$.  Assume for all $r\in[R],t\in[K]$, $\ebb_{i_{m,r,t}}[\|\nabla f(\bw_{r,t};z_{i_{m,r,t}})-\nabla F_S(\bw_{r,t})\|_2^2]\leq\sigma^2$.
  Then the following inequality holds
  \[
  \ebb_A[F_S(\bar{\bw}_{R,2})]-F_S(\bw^*)\lesssim\frac{\sigma^2}{\mu MKR}+\frac{L\log(RK)}{\mu^2KR^2}.
  \]
\end{lemma}

\begin{proof}[Proof of Theorem \ref{thm:risk-local-sg}]
By the analysis in the proof of Theorem \ref{thm:stab-local} (e.g. Eq. \eqref{local-3}), we know
\begin{align*}
  \frac{1}{n}\sum_{k=1}^{n}\ebb\big[\|\bw_{m,r,t}-\bw^{(k)}_{m,r,t}\|_2\big] &\leq \frac{2\sqrt{2L}}{nM}\sum_{r'=1}^{R}\sum_{m'=1}^{M}\sum_{t'=1}^{K}\eta_{r',t'}\ebb\Big[\sqrt{F_S(\bw_{m',r',t'})}\Big]\\
  & \lesssim \frac{\sqrt{L}}{nM}\sum_{r'=1}^{R}\sum_{m'=1}^{M}\sum_{t'=1}^{K}\frac{1}{\mu(a+(r'-1)K+t')}\\
  & \lesssim \frac{\sqrt{L}}{nM}\sum_{m'=1}^{M}\frac{\log(KR)}{\mu}\lesssim \frac{\sqrt{L}\log(KR)}{n\mu},\quad\forall r\in[R],t\in[K].
\end{align*}
Since the above inequality holds for all $ r\in[R],t\in[K]$ and $\bar{\bw}_{R,2}$ is a weighted average of $\bw_{m,r,t}$, we then get
\[
\frac{1}{n}\sum_{k=1}^{n}\ebb\big[\|\bar{\bw}_{R,2}-\bar{\bw}_{R,2}^{(k)}\|_2\big]\lesssim \frac{\sqrt{L}\log(KR)}{n\mu}
\]
and therefore
$
\ebb[F(\bar{\bw}_{R,2})-F_S(\bar{\bw}_{R,2})] \lesssim\frac{\sqrt{L}G\log(KR)}{n\mu}.
$
We combine this generalization error bound and the optimization error bound in Lemma \ref{lem:opt-local-sg} to derive
\[
\ebb[F(\bar{\bw}_{R,2})]-F(\bw^*)\lesssim \frac{G\sqrt{L}\log(KR)}{n\mu}+\frac{\sigma^2}{\mu MKR}+\frac{L\log(RK)}{\mu^2KR^2}\lesssim \frac{G\sqrt{L}\log(KR)}{n\mu},
\]
where we have used $KR\gtrsim \frac{n\sigma^2}{MG\sqrt{L}}$ and $\mu KR^2\gtrsim n\sqrt{L}/G$ in the last inequality.
\end{proof}

\section{Conclusion\label{sec:conclusion}}
We investigate the stability and generalization of minibatch SGD and local SGD, which are widely used for large-scale learning problems. While there are many discussions on the speedup of these methods for optimization, we study the linear speedup in generalization. We develop on-average stability bounds for convex, strongly convex and nonconvex problems, and show how small training errors can improve stability. For strongly convex problems, our stability bounds are independent of the iteration number, which is new for the vanilla SGD in the sense of removing the Lipschitzness assumption. Our stability analysis implies optimal excess population risk bounds with both a linear speedup w.r.t. the batch size for minibatch SGD and a linear speedup w.r.t. the number of machines for local SGD.

There are several limitations of our work.  A limitation of our work is that we do not get optimistic bounds for local SGD which are important to show the benefit of low noises. Another limitation is that
we only consider homogeneous setups in local SGD. It would be very interesting to extend the analysis to heterogeneous setups, i.e., where different local machines have different sets of examples. We will study these limitations in our future work.

\section*{Acknowledgements}
The authors are grateful to the associate editor and the anonymous reviewers for their thoughtful comments and constructive suggestions.

\setlength{\bibsep}{0.03cm}
\bibliographystyle{abbrvnat}
\small

\begin{thebibliography}{52}
\providecommand{\natexlab}[1]{#1}
\providecommand{\url}[1]{\texttt{#1}}
\expandafter\ifx\csname urlstyle\endcsname\relax
  \providecommand{\doi}[1]{doi: #1}\else
  \providecommand{\doi}{doi: \begingroup \urlstyle{rm}\Url}\fi

\bibitem[Agarwal et~al.(2009)Agarwal, Wainwright, Bartlett, and
  Ravikumar]{agarwal2009information}
A.~Agarwal, M.~J. Wainwright, P.~L. Bartlett, and P.~K. Ravikumar.
\newblock Information-theoretic lower bounds on the oracle complexity of convex
  optimization.
\newblock In \emph{Advances in Neural Information Processing Systems}, pages
  1--9, 2009.

\bibitem[Bassily et~al.(2019)Bassily, Feldman, Talwar, and
  Thakurta]{bassily2019private}
R.~Bassily, V.~Feldman, K.~Talwar, and A.~G. Thakurta.
\newblock Private stochastic convex optimization with optimal rates.
\newblock In \emph{Advances in Neural Information Processing Systems}, pages
  11279--11288, 2019.

\bibitem[Bassily et~al.(2020)Bassily, Feldman, Guzm{\'a}n, and
  Talwar]{bassily2020stability}
R.~Bassily, V.~Feldman, C.~Guzm{\'a}n, and K.~Talwar.
\newblock Stability of stochastic gradient descent on nonsmooth convex losses.
\newblock \emph{Advances in Neural Information Processing Systems},
  33:\penalty0 4381--4391, 2020.

\bibitem[Bousquet and Elisseeff(2002)]{bousquet2002stability}
O.~Bousquet and A.~Elisseeff.
\newblock Stability and generalization.
\newblock \emph{Journal of Machine Learning Research}, 2\penalty0
  (Mar):\penalty0 499--526, 2002.

\bibitem[Bousquet et~al.(2020)Bousquet, Klochkov, and
  Zhivotovskiy]{bousquet2020sharper}
O.~Bousquet, Y.~Klochkov, and N.~Zhivotovskiy.
\newblock Sharper bounds for uniformly stable algorithms.
\newblock In \emph{Conference on Learning Theory}, pages 610--626, 2020.

\bibitem[Carratino et~al.(2018)Carratino, Rudi, and
  Rosasco]{carratino2018learning}
L.~Carratino, A.~Rudi, and L.~Rosasco.
\newblock Learning with {SGD} and random features.
\newblock In \emph{Advances in Neural Information Processing Systems}, pages
  10213--10224, 2018.

\bibitem[Charles and Papailiopoulos(2018)]{charles2018stability}
Z.~Charles and D.~Papailiopoulos.
\newblock Stability and generalization of learning algorithms that converge to
  global optima.
\newblock In \emph{International Conference on Machine Learning}, pages
  744--753, 2018.

\bibitem[Chen et~al.(2023)Chen, Zheng, Long, and Su]{chen2023minimax}
S.~Chen, Q.~Zheng, Q.~Long, and W.~J. Su.
\newblock Minimax estimation for personalized federated learning: An
  alternative between fedavg and local training?
\newblock \emph{Journal of Machine Learning Research}, 24\penalty0
  (262):\penalty0 1--59, 2023.

\bibitem[Chen et~al.(2018)Chen, Jin, and Yu]{chen2018stability}
Y.~Chen, C.~Jin, and B.~Yu.
\newblock Stability and convergence trade-off of iterative optimization
  algorithms.
\newblock \emph{arXiv preprint arXiv:1804.01619}, 2018.

\bibitem[Christmann et~al.(2018)Christmann, Xiang, and
  Zhou]{christmann2018total}
A.~Christmann, D.~Xiang, and D.-X. Zhou.
\newblock Total stability of kernel methods.
\newblock \emph{Neurocomputing}, 289:\penalty0 101--118, 2018.

\bibitem[Cotter et~al.(2011)Cotter, Shamir, Srebro, and
  Sridharan]{cotter2011better}
A.~Cotter, O.~Shamir, N.~Srebro, and K.~Sridharan.
\newblock Better mini-batch algorithms via accelerated gradient methods.
\newblock \emph{Advances in Neural Information Processing Systems},
  24:\penalty0 1647--1655, 2011.

\bibitem[Dekel et~al.(2012)Dekel, Gilad-Bachrach, Shamir, and
  Xiao]{dekel2012optimal}
O.~Dekel, R.~Gilad-Bachrach, O.~Shamir, and L.~Xiao.
\newblock Optimal distributed online prediction using mini-batches.
\newblock \emph{Journal of Machine Learning Research}, 13\penalty0 (1), 2012.

\bibitem[Fan and Lei(2024)]{fan2024high}
J.~Fan and Y.~Lei.
\newblock High-probability generalization bounds for pointwise uniformly stable
  algorithms.
\newblock \emph{Applied and Computational Harmonic Analysis}, 70:\penalty0
  101632, 2024.

\bibitem[Feldman and Vondrak(2019)]{feldman2019high}
V.~Feldman and J.~Vondrak.
\newblock High probability generalization bounds for uniformly stable
  algorithms with nearly optimal rate.
\newblock In \emph{Conference on Learning Theory}, pages 1270--1279, 2019.

\bibitem[Ghadimi et~al.(2016)Ghadimi, Lan, and Zhang]{ghadimi2016mini}
S.~Ghadimi, G.~Lan, and H.~Zhang.
\newblock Mini-batch stochastic approximation methods for nonconvex stochastic
  composite optimization.
\newblock \emph{Mathematical Programming}, 155\penalty0 (1-2):\penalty0
  267--305, 2016.

\bibitem[Gower et~al.(2021)Gower, Sebbouh, and Loizou]{gower2021sgd}
R.~Gower, O.~Sebbouh, and N.~Loizou.
\newblock {SGD} for structured nonconvex functions: Learning rates,
  minibatching and interpolation.
\newblock In \emph{International Conference on Artificial Intelligence and
  Statistics}, pages 1315--1323. PMLR, 2021.

\bibitem[Guo et~al.(2023)Guo, Guo, and Shi]{guo2023capacity}
X.~Guo, Z.-C. Guo, and L.~Shi.
\newblock Capacity dependent analysis for functional online learning
  algorithms.
\newblock \emph{Applied and Computational Harmonic Analysis}, 67:\penalty0
  101567, 2023.

\bibitem[Hardt et~al.(2016)Hardt, Recht, and Singer]{hardt2016train}
M.~Hardt, B.~Recht, and Y.~Singer.
\newblock Train faster, generalize better: Stability of stochastic gradient
  descent.
\newblock In \emph{International Conference on Machine Learning}, pages
  1225--1234, 2016.

\bibitem[Hu et~al.(2020)Hu, Wu, and Zhou]{hu2020distributed}
T.~Hu, Q.~Wu, and D.-X. Zhou.
\newblock Distributed kernel gradient descent algorithm for minimum error
  entropy principle.
\newblock \emph{Applied and Computational Harmonic Analysis}, 49\penalty0
  (1):\penalty0 229--256, 2020.

\bibitem[Keskar et~al.(2017)Keskar, Mudigere, Nocedal, Smelyanskiy, and
  Tang]{keskar2017large}
N.~S. Keskar, D.~Mudigere, J.~Nocedal, M.~Smelyanskiy, and P.~T.~P. Tang.
\newblock On large-batch training for deep learning: Generalization gap and
  sharp minima.
\newblock In \emph{International Conference on Learning Representations}, 2017.

\bibitem[Khaled et~al.(2020)Khaled, Mishchenko, and
  Richt{\'a}rik]{khaled2020tighter}
A.~Khaled, K.~Mishchenko, and P.~Richt{\'a}rik.
\newblock Tighter theory for local {SGD} on identical and heterogeneous data.
\newblock In \emph{International Conference on Artificial Intelligence and
  Statistics}, pages 4519--4529. PMLR, 2020.

\bibitem[Kone{\v{c}}n{\`y} et~al.(2016)Kone{\v{c}}n{\`y}, McMahan, Yu,
  Richt{\'a}rik, Suresh, and Bacon]{konevcny2016federated}
J.~Kone{\v{c}}n{\`y}, H.~B. McMahan, F.~X. Yu, P.~Richt{\'a}rik, A.~T. Suresh,
  and D.~Bacon.
\newblock Federated learning: Strategies for improving communication
  efficiency.
\newblock \emph{arXiv preprint arXiv:1610.05492}, 2016.

\bibitem[Koren et~al.(2022)Koren, Livni, Mansour, and Sherman]{koren2022benign}
T.~Koren, R.~Livni, Y.~Mansour, and U.~Sherman.
\newblock Benign underfitting of stochastic gradient descent.
\newblock In \emph{Advances in Neural Information Processing Systems}, 2022.

\bibitem[Kuzborskij and Lampert(2018)]{kuzborskij2018data}
I.~Kuzborskij and C.~Lampert.
\newblock Data-dependent stability of stochastic gradient descent.
\newblock In \emph{International Conference on Machine Learning}, pages
  2820--2829, 2018.

\bibitem[Lei and Ying(2020)]{lei2020fine}
Y.~Lei and Y.~Ying.
\newblock Fine-grained analysis of stability and generalization for stochastic
  gradient descent.
\newblock In \emph{International Conference on Machine Learning}, pages
  5809--5819, 2020.

\bibitem[Lei and Ying(2021)]{lei2021sharper}
Y.~Lei and Y.~Ying.
\newblock Sharper generalization bounds for learning with gradient-dominated
  objective functions.
\newblock In \emph{International Conference on Learning Representations}, 2021.

\bibitem[Li et~al.(2022)Li, Li, and Chi]{li2022destress}
B.~Li, Z.~Li, and Y.~Chi.
\newblock Destress: Computation-optimal and communication-efficient
  decentralized nonconvex finite-sum optimization.
\newblock \emph{SIAM Journal on Mathematics of Data Science}, 4\penalty0
  (3):\penalty0 1031--1051, 2022.

\bibitem[Li et~al.(2014)Li, Zhang, Chen, and Smola]{li2014efficient}
M.~Li, T.~Zhang, Y.~Chen, and A.~J. Smola.
\newblock Efficient mini-batch training for stochastic optimization.
\newblock In \emph{ACM SIGKDD international conference on Knowledge discovery
  and data mining}, pages 661--670, 2014.

\bibitem[Lin et~al.(2020)Lin, Rudi, Rosasco, and Cevher]{lin2020optimal}
J.~Lin, A.~Rudi, L.~Rosasco, and V.~Cevher.
\newblock Optimal rates for spectral algorithms with least-squares regression
  over hilbert spaces.
\newblock \emph{Applied and Computational Harmonic Analysis}, 48\penalty0
  (3):\penalty0 868--890, 2020.

\bibitem[Lin et~al.(2017)Lin, Guo, and Zhou]{lin2017distributed}
S.-B. Lin, X.~Guo, and D.-X. Zhou.
\newblock Distributed learning with regularized least squares.
\newblock \emph{The Journal of Machine Learning Research}, 18\penalty0
  (1):\penalty0 3202--3232, 2017.

\bibitem[Lin et~al.(2019)Lin, Stich, Patel, and Jaggi]{lin2019don}
T.~Lin, S.~U. Stich, K.~K. Patel, and M.~Jaggi.
\newblock Don't use large mini-batches, use local {SGD}.
\newblock In \emph{International Conference on Learning Representations}, 2019.

\bibitem[Mc{M}ahan et~al.(2017)Mc{M}ahan, Moore, Ramage, Hampson, and
  Arcas]{mcmahan2017communication}
B.~Mc{M}ahan, E.~Moore, D.~Ramage, S.~Hampson, and B.~A. Arcas.
\newblock Communication-efficient learning of deep networks from decentralized
  data.
\newblock In \emph{International Conference on Artificial Intelligence and
  Statistics}, pages 1273--1282, 2017.

\bibitem[Mohri et~al.(2019)Mohri, Sivek, and Suresh]{mohri2019agnostic}
M.~Mohri, G.~Sivek, and A.~T. Suresh.
\newblock Agnostic federated learning.
\newblock In \emph{International Conference on Machine Learning}, pages
  4615--4625. PMLR, 2019.

\bibitem[Mou et~al.(2018)Mou, Wang, Zhai, and Zheng]{mou2018generalization}
W.~Mou, L.~Wang, X.~Zhai, and K.~Zheng.
\newblock Generalization bounds of {SGLD} for non-convex learning: Two
  theoretical viewpoints.
\newblock In \emph{Conference on Learning Theory}, pages 605--638, 2018.

\bibitem[M{\"u}cke et~al.(2019)M{\"u}cke, Neu, and Rosasco]{mucke2019beating}
N.~M{\"u}cke, G.~Neu, and L.~Rosasco.
\newblock Beating {SGD} saturation with tail-averaging and minibatching.
\newblock In \emph{Advances in Neural Information Processing Systems}, pages
  12568--12577, 2019.

\bibitem[Schliserman and Koren(2022)]{schliserman2022stability}
M.~Schliserman and T.~Koren.
\newblock Stability vs implicit bias of gradient methods on separable data and
  beyond.
\newblock In \emph{Conference on Learning Theory}, pages 3380--3394. PMLR,
  2022.

\bibitem[Shalev-Shwartz et~al.(2010)Shalev-Shwartz, Shamir, Srebro, and
  Sridharan]{shalev2010learnability}
S.~Shalev-Shwartz, O.~Shamir, N.~Srebro, and K.~Sridharan.
\newblock Learnability, stability and uniform convergence.
\newblock \emph{Journal of Machine Learning Research}, 11\penalty0
  (Oct):\penalty0 2635--2670, 2010.

\bibitem[Shamir and Srebro(2014)]{shamir2014distributed}
O.~Shamir and N.~Srebro.
\newblock Distributed stochastic optimization and learning.
\newblock In \emph{Annual Allerton Conference on Communication, Control, and
  Computing}, pages 850--857. IEEE, 2014.

\bibitem[Spiridonoff et~al.(2021)Spiridonoff, Olshevsky, and
  Paschalidis]{spiridonoff2021communication}
A.~Spiridonoff, A.~Olshevsky, and Y.~Paschalidis.
\newblock Communication-efficient {SGD}: From local {SGD} to one-shot
  averaging.
\newblock In \emph{Advances in Neural Information Processing Systems},
  volume~34, pages 24313--24326, 2021.

\bibitem[Srebro et~al.(2010)Srebro, Sridharan, and
  Tewari]{srebro2010smoothness}
N.~Srebro, K.~Sridharan, and A.~Tewari.
\newblock Smoothness, low noise and fast rates.
\newblock In \emph{Advances in Neural Information Processing Systems}, pages
  2199--2207, 2010.

\bibitem[Stich(2018)]{stich2018local}
S.~U. Stich.
\newblock Local {SGD} converges fast and communicates little.
\newblock In \emph{International Conference on Learning Representations}, 2018.

\bibitem[Sun et~al.(2024)Sun, Niu, and Wei]{sun2023understanding}
Z.~Sun, X.~Niu, and E.~Wei.
\newblock Understanding generalization of federated learning via stability:
  Heterogeneity matters.
\newblock In \emph{International Conference on Artificial Intelligence and
  Statistics}, pages 676--684, 2024.

\bibitem[Wang et~al.(2022)Wang, Lei, Ying, and Zhang]{wang2022differentially}
P.~Wang, Y.~Lei, Y.~Ying, and H.~Zhang.
\newblock Differentially private {SGD} with non-smooth losses.
\newblock \emph{Applied and Computational Harmonic Analysis}, 56:\penalty0
  306--336, 2022.

\bibitem[Woodworth et~al.(2020{\natexlab{a}})Woodworth, Patel, Stich, Dai,
  Bullins, Mc{M}ahan, Shamir, and Srebro]{woodworth2020local}
B.~Woodworth, K.~K. Patel, S.~Stich, Z.~Dai, B.~Bullins, B.~Mc{M}ahan,
  O.~Shamir, and N.~Srebro.
\newblock Is local {SGD} better than minibatch {SGD}?
\newblock In \emph{International Conference on Machine Learning}, pages
  10334--10343. PMLR, 2020{\natexlab{a}}.

\bibitem[Woodworth et~al.(2020{\natexlab{b}})Woodworth, Patel, and
  Srebro]{woodworth2020minibatch}
B.~E. Woodworth, K.~K. Patel, and N.~Srebro.
\newblock Minibatch vs local {SGD} for heterogeneous distributed learning.
\newblock In \emph{Advances in Neural Information Processing Systems},
  volume~33, pages 6281--6292, 2020{\natexlab{b}}.

\bibitem[Yao et~al.(2007)Yao, Rosasco, and Caponnetto]{yao2007early}
Y.~Yao, L.~Rosasco, and A.~Caponnetto.
\newblock On early stopping in gradient descent learning.
\newblock \emph{Constructive Approximation}, 26\penalty0 (2):\penalty0
  289--315, 2007.

\bibitem[Yin et~al.(2018)Yin, Pananjady, Lam, Papailiopoulos, Ramchandran, and
  Bartlett]{yin2018gradient}
D.~Yin, A.~Pananjady, M.~Lam, D.~Papailiopoulos, K.~Ramchandran, and
  P.~Bartlett.
\newblock Gradient diversity: a key ingredient for scalable distributed
  learning.
\newblock In \emph{International Conference on Artificial Intelligence and
  Statistics}, pages 1998--2007. PMLR, 2018.

\bibitem[Ying and Zhou(2017)]{ying2017unregularized}
Y.~Ying and D.-X. Zhou.
\newblock Unregularized online learning algorithms with general loss functions.
\newblock \emph{Applied and Computational Harmonic Analysis}, 42\penalty0
  (2):\penalty0 224--244, 2017.

\bibitem[Yu et~al.(2019)Yu, Yang, and Zhu]{yu2019parallel}
H.~Yu, S.~Yang, and S.~Zhu.
\newblock Parallel restarted {SGD} with faster convergence and less
  communication: Demystifying why model averaging works for deep learning.
\newblock In \emph{AAAI Conference on Artificial Intelligence}, pages
  5693--5700, 2019.

\bibitem[Yun et~al.(2022)Yun, Rajput, and Sra]{yun2021minibatch}
C.~Yun, S.~Rajput, and S.~Sra.
\newblock Minibatch vs local {SGD} with shuffling: Tight convergence bounds and
  beyond.
\newblock In \emph{International Conference on Learning Representations}, 2022.

\bibitem[Zhang et~al.(2022)Zhang, Zhang, Bald, Pingali, Chen, and
  Goswami]{zhang2022stability}
Y.~Zhang, W.~Zhang, S.~Bald, V.~Pingali, C.~Chen, and M.~Goswami.
\newblock Stability of {SGD}: Tightness analysis and improved bounds.
\newblock In \emph{Uncertainty in Artificial Intelligence}, pages 2364--2373.
  PMLR, 2022.

\bibitem[Zinkevich et~al.(2010)Zinkevich, Weimer, Smola, and
  Li]{zinkevich2010parallelized}
M.~A. Zinkevich, M.~Weimer, A.~Smola, and L.~Li.
\newblock Parallelized stochastic gradient descent.
\newblock In \emph{Advances in Neural Information Processing Systems}, pages
  2595--2603, 2010.

\end{thebibliography}

\end{document}